\documentclass{article}

\usepackage{iclr2027_conference,times}
\usepackage{amsmath,amsfonts,bm}

\def\eqref#1{equation~\ref{#1}}
\def\1{\bm{1}}

\DeclareMathAlphabet{\mathsfit}{\encodingdefault}{\sfdefault}{m}{sl}
\SetMathAlphabet{\mathsfit}{bold}{\encodingdefault}{\sfdefault}{bx}{n}

\newcommand{\E}{\mathbb{E}}

\newcommand{\sigmoid}{\sigma}

\newcommand{\KL}{D_{\mathrm{KL}}}

\usepackage{amsmath,amssymb,amsthm,mathtools}
\usepackage{booktabs}
\usepackage{array,colortbl}
\usepackage{multirow}
\usepackage{graphicx}
\usepackage{subcaption}
\usepackage{microtype}
\usepackage{placeins}
\usepackage{wrapfig}
\usepackage{float}
\usepackage{needspace}
\usepackage{enumitem}
\usepackage{algorithm}
\usepackage[noEnd=true,indLines=true,italicComments=false,commentColor=black]{algpseudocodex}
\usepackage{hyperref}
\usepackage{url}
\usepackage[nameinlink]{cleveref}
\crefname{figure}{Fig.}{Figs.}
\crefname{equation}{Eq.}{Eqs.}
\Crefname{equation}{Equation}{Equations}
\crefname{algorithm}{Alg.}{Algs.}
\Crefname{algorithm}{Algorithm}{Algorithms}
\algrenewcommand{\algorithmicrequire}{\textbf{Input:}}
\definecolor{linkblue}{RGB}{0,68,136}
\definecolor{codepink}{RGB}{192,38,118}
\definecolor{oursblue}{RGB}{241,247,253}
\definecolor{sectiongray}{RGB}{242,242,242}
\definecolor{myred}{RGB}{154,0,0}
\definecolor{darkblue}{rgb}{0,0,0.5}
\definecolor{lightblue}{RGB}{11,83,150}
\definecolor{loopblue}{RGB}{230,240,237}
\newcolumntype{O}{>{\columncolor{oursblue}}c}
\hypersetup{
  colorlinks=true,
  linkcolor=linkblue,
  citecolor=linkblue,
  urlcolor=linkblue
}

\title{\centering ALoDLM: Adaptively Looped Diffusion\\ Language Models}

\author{%
  \begin{minipage}[t]{\dimexpr\textwidth-2\tabcolsep\relax}
  \centering
  \bfseries Liancheng Fang$^{1,*,\ddagger}$ \quad
  Zhuowei Li$^{2,*,\dagger}$ \quad
  Youngeun Kim$^{3,\ddagger}$ \quad Tianchen Zhao$^{2}$\\
  Rajat Koner$^{2}$ \quad Jiaye Wu$^{2}$ \quad
  Linghan Xu$^{2}$ \quad Xuanbai Chen$^{2}$\\
  Xiang Xu$^{2}$ \quad Zheng Zhang$^{2}$ \quad
  Jakub Zablocki$^{2}$ \quad Nishant Sankaran$^{2}$ \quad
  Yifan Xing$^{2}$\\[0.8em]
  \textnormal{$^{1}$University of Illinois Chicago \quad
  $^{2}$Amazon AGI \quad $^{3}$Korea University}\\[0.6em]
  \begingroup
  \hypersetup{urlcolor=codepink}
  \urlstyle{tt}
  \textnormal{Project: \url{https://alo-dlm.github.io}}\\
  \textnormal{Code: \url{https://github.com/amazon-science/ALoDLM}}\\
  \textnormal{Model: \url{https://huggingface.co/amazon/ALoDLM-8B}}
  \endgroup
  \end{minipage}%
}
\date{}

\newcommand{\modelname}{\texttt{ALoDLM}}

\newcommand{\stopgrad}{\operatorname{sg}}
\newcommand{\Unif}{\operatorname{Unif}}
\newcommand{\emb}{\operatorname{Emb}}

\newtheorem{proposition}{Proposition}
\newtheorem{theorem}[proposition]{Theorem}

\crefname{proposition}{Proposition}{Propositions}
\crefname{theorem}{Theorem}{Theorems}
\theoremstyle{definition}
\newtheorem{definition}{Definition}[section]
\crefname{definition}{Definition}{Definitions}
\theoremstyle{remark}

\iclrfinalcopy

\begin{document}

\addtocontents{toc}{\protect\setcounter{tocdepth}{-1}}

\maketitle
\fancyhead{}
\renewcommand{\headrulewidth}{0pt}

\begingroup
\renewcommand{\thefootnote}{\fnsymbol{footnote}}
\footnotetext[1]{Equal contribution.}
\footnotetext[2]{Project lead.}
\footnotetext[3]{Work done at Amazon.}
\footnotetext[0]{Contacts:
\href{mailto:lfang87@uic.edu}{\nolinkurl{lfang87@uic.edu}} \quad
\href{mailto:zhuoweli@amazon.com}{\nolinkurl{zhuoweli@amazon.com}}.}
\endgroup

\vspace{-10pt}
\begin{abstract}
  Diffusion language models (DLMs) enable fast generation by predicting multiple tokens in parallel, yet their practical adoption remains hindered by a persistent quality gap relative to comparably sized autoregressive (AR) models. We attribute this gap to a \textit{computation--difficulty mismatch}: within a partially observed sequence, some unknown tokens are readily predictable, while others require substantially more computation to resolve. Existing DLMs, however, apply uniform computational depth to every unknown position at each denoising step. We introduce \textbf{\modelname{}}, which replaces this uniform computation with token-adaptive latent recurrence. At each denoising step, \modelname{} iteratively refines representations in latent space, allocating computation based on token difficulty. Tokens ready to commit are fed back as discrete context, while unresolved tokens retain and refine their latent states through additional recurrent passes. To learn both token prediction and computation allocation end-to-end, we formulate token-wise computation schedules as latent variables and derive a conditional negative evidence lower bound (NELBO). We train \modelname{} at 1.7B and 8B parameter scales. Across eleven benchmarks, \modelname{} outperforms all evaluated DLMs and the corresponding AR baselines in average benchmark score at both scales.
  Importantly, \modelname{} combines superior generation quality with fast parallel decoding, establishing a strong quality--efficiency trade-off among all evaluated autoregressive and diffusion models under optimized inference engines.
\end{abstract}

\vspace{-10pt}
\begingroup
\setlength{\intextsep}{1pt}
\begin{figure}[H]
    \centering
    \setlength{\abovecaptionskip}{3pt}
    \setlength{\belowcaptionskip}{0pt}
    \begin{minipage}[c]{0.49\linewidth}
        \centering
        \includegraphics[
            width=\linewidth,
            height=1.0\linewidth,
            keepaspectratio
        ]{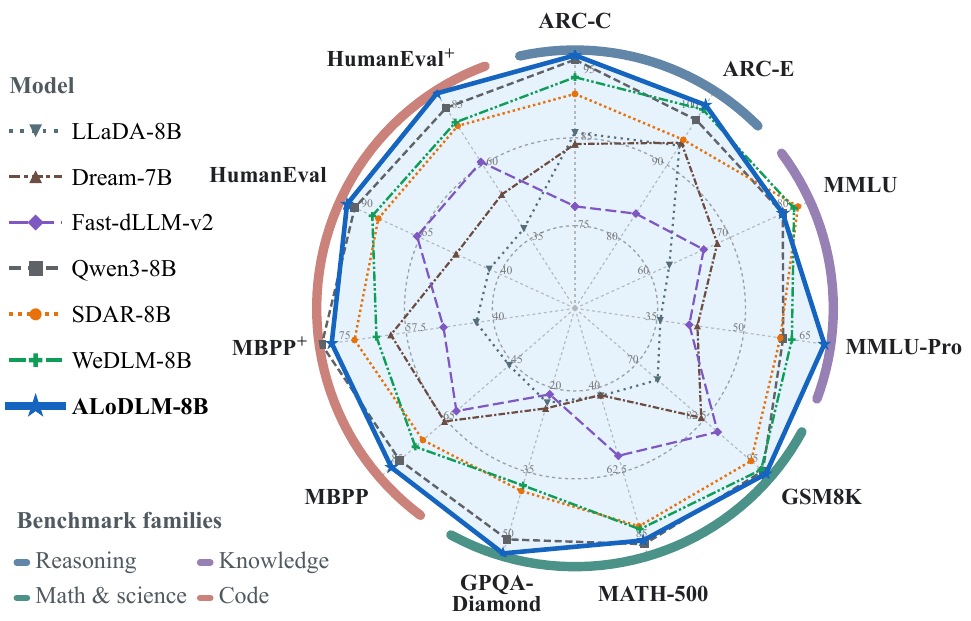}
    \end{minipage}%
    \hfill
    \begin{minipage}[c]{0.49\linewidth}
        \centering
        \includegraphics[
            width=\linewidth,
            height=0.7\linewidth,
            keepaspectratio
        ]{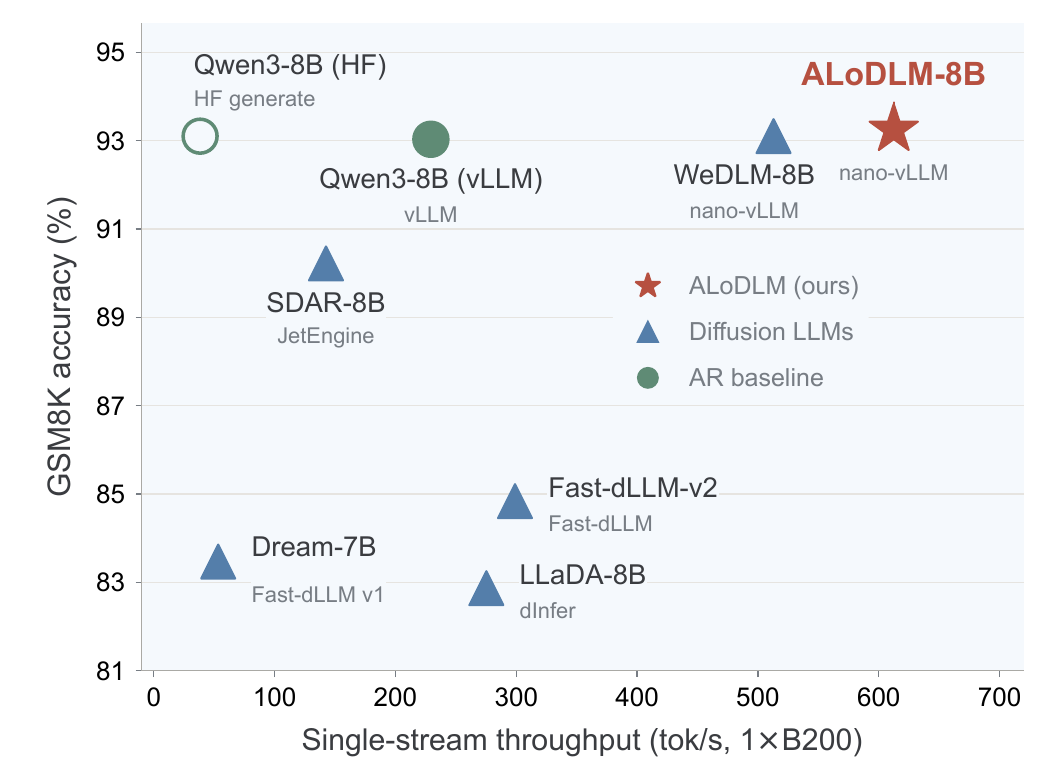}
    \end{minipage}
    \caption{
    \textbf{\modelname{} combines strong performance with fast inference.}
    \textbf{Left:} \modelname{}-8B outperforms all baselines on average
    across 11 reasoning, coding, and knowledge benchmarks.
    \textbf{Right:} On GSM8K, \modelname{}-8B delivers
    $\sim 2.7\times$ the throughput of vLLM-served Qwen3-8B
    at comparable or higher accuracy, while outperforming the
    diffusion baselines shown in both accuracy and throughput.
    }
    \vspace{-10pt}
    \label{fig:teaser-8b}
\end{figure}
\endgroup

\section{Introduction}
\label{sec:introduction}
Autoregressive (AR) large language models (LLMs) have demonstrated remarkable capabilities across diverse tasks~\citep{radford2018improving,raffel2020exploring,brown2020language}. However, their sequential, token-by-token decoding imposes a severe latency bottleneck on generation~\citep{pmlr-v202-leviathan23a,pmlr-v235-cai24b}. Diffusion language models (DLMs)~\citep{austin2021structured,lou2023sedd,shi2024simplified,ou2024absorbing,sahoo2024mdlm,nie2025llada,ye2025dream,song2025seed} offer a promising alternative: they iteratively denoise corrupted sequences, resolving multiple masked positions in parallel to enable substantially faster generation~\citep{zhang2025survey}. Despite this potential, their practical adoption remains hindered by a persistent quality gap relative to comparably sized AR models. Existing approaches mitigate this gap either by reintroducing autoregressive structures via semi-AR designs like block diffusion~\citep{arriola2025block,wu2025fastdllmv2,cheng2025sdar,liu2025wedlm,Fu2026NemotronLabsDiffusionAT}, or by relegating DLMs to drafting candidates for subsequent AR verification~\citep{Liu2025TiDARTI,Chen2026DFlashBD,yu2026introspective,sahoo2026unlocking}. We instead ask a more fundamental question:
\begin{center}
\textbf{\emph{What is the root cause underlying the generation quality gap in DLMs?}}
\end{center}

We hypothesize that a key bottleneck in DLM's generation quality is a \textit{computation--difficulty mismatch}. During a single denoising step, masked positions can vary significantly in prediction difficulty due to intrinsic uncertainty~\citep{hullermeier2021aleatoric}, multihop reasoning requirements~\citep{Huang2025TransformersPL}, or dependencies on unresolved tokens~\citep{wu2025fastdllmtrainingfreeaccelerationdiffusion}. Despite this variation, standard DLMs apply the same fixed-depth denoiser to all masked positions, training each to predict its target directly from a corrupted context. This uniform allocation wastes computation on easy predictions while starving difficult ones of necessary compute. AR generation largely avoids this within-step allocation problem by predicting a single next token conditioned on a fully resolved prefix.
Confidence-based decoding~\citep{nie2025llada} partially alleviates the problem by deferring low-confidence predictions to later denoising steps. However, the deferred tokens ultimately receive the same fixed compute since the latent computation performed for an unresolved token is discarded, and the token is recomputed from its mask embedding at the next step. 
This limitation motivates our central principle: DLMs should maintain
and refine persistent latent states for challenging tokens within each
denoising step, allowing computation to scale dynamically with
prediction difficulty.

In this work, we propose \textbf{\modelname}, a novel class of DLMs designed to overcome this bottleneck by adaptively allocating computation across tokens during mask recovery.
Our first contribution is a looped architecture with token-adaptive recurrent depth. In this design, easy tokens commit early to provide resolved context, while difficult tokens undergo additional recurrent passes to continually refine their persistent latent states. Building on this architecture, we formulate the token-wise computation schedule as a latent variable and derive a unified objective that jointly optimizes the base language model and the compute-allocation policy.
Finally, we derive an unbiased stochastic gradient estimator equipped with variance reduction techniques for stable end-to-end training.

Experimental results demonstrate that \modelname{} achieves state-of-the-art generation quality while maintaining fast parallel decoding.
We train \modelname{} at 1.7B and 8B scales through direct SFT of Qwen3 backbones on a 5B-token corpus. 
In quality evaluations across eleven benchmarks
spanning general reasoning, mathematics, and coding, 
\modelname{} achieves average scores of 65.5 and 80.3, surpassing all evaluated DLMs and the corresponding AR baselines. 
Furthermore, with parallel decoding on GSM8K, \modelname{}-8B delivers approximately $2.7\times$ the speed of vLLM-served Qwen3-8B at comparable or higher accuracy, demonstrating a strong quality--efficiency trade-off.
Finally, our analyses show that recurrent depth provides an effective test-time scaling axis and that \modelname{} learns token-dependent refinement preferences without explicit difficulty labels. Our main contributions are:
\begin{itemize}[
    leftmargin=1.25em,
    labelsep=0.45em,
    topsep=0pt,
    itemsep=2pt,
    parsep=0pt,
    partopsep=0pt]
    \item \textbf{Token-adaptive looped architecture.}
    We introduce a looped architecture with dynamic recurrent depth. Easy tokens commit early to act as resolved context, while difficult tokens continue to refine their accumulated latent states through additional recurrent passes.

    \item \textbf{Principled training objective.} We formulate token-wise exit schedules as latent
    variables and derive a conditional ELBO to jointly optimize token prediction and computation
    allocation. To stabilize training, we develop an unbiased single-trajectory gradient estimator
    coupled with effective variance reduction techniques.

    \item \textbf{Scaling and performance.} 
    We scale \modelname{} to 8B parameters, making it  to the best of our knowledge, the largest looped DLM to date. It achieves state-of-the-art average benchmark performance among all evaluated DLMs and outperforms its Qwen3-based AR counterparts, all while maintaining fast parallel generation. 
\end{itemize}

\section{Related Work}
\label{sec:related-work-main}

\textbf{Looped Transformers.}
Looped (or recurrent-depth) transformers repeatedly apply layers or blocks of layers, effectively
increasing network depth without expanding the parameter
count~\citep{dehghani2019universal,xue2022go,zhu2025beyond,bae2025mixture,lab2026nanbeige4,wang2026smelt,gao2026loop,jeddi2026loopformer,movahedi2026fixed,
sun2026universal}. This looping mechanism serves as a new scaling axis, yielding performance gains
on both algorithmic~\citep{giannou2023looped,yang2024looped} and reasoning
tasks~\citep{saunshi2025latent,geiping2025recurrent,zhu2025beyond,huang2026equilibrium}. A related
line of research leverages this recurrence for \emph{adaptive computation}, where the primary
challenge is learning a policy to determine when to halt. ACT~\citep{graves2016adaptive} and
Universal Transformers~\citep{dehghani2019universal} learn halting through continuous weighting of
intermediate states or outputs. PonderNet~\citep{banino2021pondernet} and Ouro~\citep{zhu2025ouro}
instead optimize the expected prediction loss by explicitly summing over candidate exit depths,
weighted by learned halting probabilities. In \modelname{}, token commitments change the context of
unresolved positions, making exact summation over joint exit schedules intractable.

\textbf{Looped Diffusion Language Models.}
Diffusion language models (DLMs) generate text through parallel token denoising and are typically trained with an ELBO-derived cross-entropy objective over masked positions \citep{sahoo2024mdlm,shi2024simplified,ou2024absorbing}. Recent work has integrated looping mechanisms into DLMs to improve training efficiency \citep{lee2026loopmdm}, parameter efficiency \citep{carballocastro2026recursive}, and step-wise adaptive computation \citep{miele2026fixedpoint}. Crucially, they apply a uniform recurrent depth across all token positions within each denoising step, lacking the flexibility to let individual tokens exit latent refinement adaptively. 
While FReDA~\citep{sun2026forwardfree} supports token-level adaptive commitment, it iteratively revises an explicit token draft rather than propagating an internal latent state across loops. In contrast, \modelname{} enables token-wise adaptive latent depth: easy tokens commit at earlier inner passes to provide their embeddings as context, while harder positions retain their evolving hidden states for continued refinement. For other related topics such as \emph{self-conditioning}, we refer the reader to Appendix \ref{sec:related-work} for an extended discussion.

\FloatBarrier
\section{Method}
\label{sec:method}
We introduce \modelname{}, a family of DLMs with token-adaptive recurrent
computation. We first describe its inference procedure in
\S\ref{sec:inference}, then derive the training objective and present
techniques for gradient estimation and variance reduction in
\S\ref{sec:joint-elbo}.

\begin{figure}[t]
    \centering
    \includegraphics[width=\textwidth]{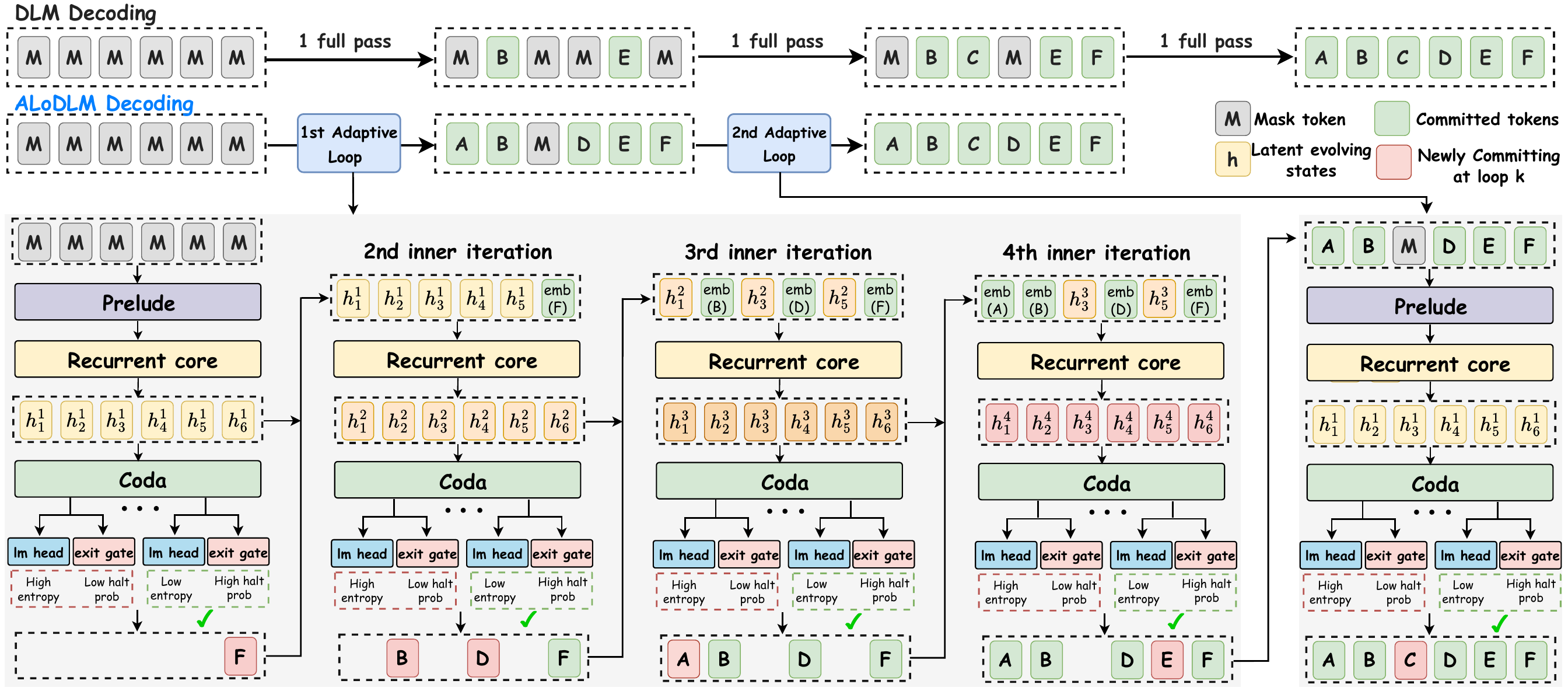}
    \caption{
    \textbf{Overview of \modelname{}.}
    Standard DLMs apply a fixed-depth denoiser at each denoising step.
    \modelname{} introduces token-adaptive recurrence within each step:
    confident tokens commit early and provide discrete context, while
    unresolved tokens retain and refine their latent states through
    additional recurrent passes.}
    \vspace{-10 pt}
    \label{fig:adadlm-overview}
\end{figure}

\subsection{Inference}
\label{sec:inference}

\textbf{Overview.}
At each denoising step (the \emph{outer loop}), a conventional fixed-depth
DLM predicts masked tokens in parallel using the same denoiser depth.
\modelname{} augments each step with an \emph{inner loop} that allows
tokens to commit at different recurrent depths. Committed tokens provide
discrete context for subsequent predictions, while unresolved positions
preserve their latent states for further refinement. Predictive confidence
controls token commitment, and a learned halting policy determines when
to end the inner loop. Algorithm~\ref{alg:decode} and
Figure~\ref{fig:adadlm-overview} summarize the procedure.

\textbf{Architecture.}
We adopt a standard Transformer architecture, such as Qwen3~\citep{yang2025qwen3}, and partition the
model architecture into three components:
a $\mathsf{Prelude}$ consisting of the token embedding layer and
optional prefix Transformer blocks;
a $\mathsf{Recurrent \ Core}$ comprising the intermediate Transformer blocks; and a $\mathsf{Coda}$ comprising any remaining suffix
Transformer blocks.

\textbf{Adaptive recurrent depth.}
At the beginning of each denoising step, let
$\mathbf x_t:=(x_1,\ldots,x_N)$ denote the corrupted input and
$\mathcal U:=\{i:x_i=\mathtt{[MASK]}\}$ its initial set of unresolved
positions. The $\mathsf{Prelude}$ produces the initial recurrent state
${\color{myred}h^{(0)}}=\mathsf{Prelude}(\mathbf x_t)
\in\mathbb R^{N\times d}$, where $d$ is the hidden dimension.
At recurrent pass $s$, the $\mathsf{Recurrent \ Core}$ updates the current state
${\color{myred}h^{(s-1)}}$ to produce an intermediate state ${\color{lightblue}\widetilde h^{(s)}}$,
which the $\mathsf{Coda}$ maps to a readout state:
\begin{equation}
    {\color{lightblue}\widetilde h^{(s)}}
    =\mathsf{Recurrent \ Core}({\color{myred}h^{(s-1)}}),
    \qquad
    r^{(s)}
    =\mathsf{Coda}({\color{lightblue}\widetilde h^{(s)}}),
    \qquad s=1,\ldots,K,
    \label{eq:looped-forward}
\end{equation}
where $K$ is the maximum recurrent depth. 
Finally, each readout state $r_i^{(s)}$ is processed by two heads in parallel. The unembedding head
$\mathsf{LMHead}:\mathbb R^d\to\mathbb R^{|\mathcal V|}$ produces the vocabulary distribution, while
an additional head $\mathsf{ExitGate}:\mathbb R^d\to\mathbb R$ produces a scalar logit whose sigmoid
defines the per-pass halting probability, \textit{i.e.}, the probability of halting at pass \(s\), conditioned on not having halted earlier:
\begin{equation}
    p_i^{(s)}
    =\operatorname{softmax}\!\left(\mathsf{LMHead}(r_i^{(s)})\right),
    \qquad
    \lambda_i^{(s)}
    =
    \begin{cases}
        \sigmoid\!\left(
            \mathsf{ExitGate}(\stopgrad[r_i^{(s)}])
        \right), & s<K,\\
        1, & s=K.
    \end{cases}
    \label{eq:hazard}
\end{equation}
Here, $\sigmoid$ denotes the sigmoid function, and $\stopgrad$ stops gradients from the exit gate into
the backbone. We set $\lambda_i^{(K)}=1$ to force all remaining tokens to exit at depth $K$, ensuring that the exit-depth probabilities sum to one. The two heads serve complementary roles: we use predictive confidence from the $\mathsf{LMHead}$ to determine which tokens are ready to commit, while the $\mathsf{ExitGate}$ controls whether to continue refining unresolved positions. We present the detailed procedure below.

\begin{wrapfigure}{R}{0.48\textwidth}
  \vspace{-10pt}
  \begin{minipage}{\linewidth}
    \begin{algorithm}[H]
      \caption{\modelname{} decoding}
      \label{alg:decode}
      \begingroup
        \small
        \algrenewcommand{\algorithmicindent}{1em}
        \algrenewcommand{\alglinenumber}[1]{%
          \makebox[0pt][r]{\footnotesize #1:}}
        \begin{algorithmic}[1]
          \Require masked sequence $\mathbf x$, maximum depth $K$,
          entropy threshold $\tau$, cumulative halt threshold $q$
          \BeginBox[fill=loopblue,draw=none]
            \While{$\mathbf x$ contains $\mathtt{[MASK]}$}
              \Comment{\textbf{Outer loop}}
            \EndBox
            \State $\mathcal U\gets\{i:x_i=\mathtt{[MASK]}\}$
              \label{alg:decode-active}
            \State $h\gets\mathsf{Prelude}(\mathbf x)$;
              $a\gets0$; $\mathcal C\gets\varnothing$
              \label{alg:decode-init}
            \BeginBox[fill=loopblue,draw=none]
              \For{$s=1,\ldots,K$}
                \label{alg:decode-inner}
                \Comment{\textbf{Inner loop}}
              \EndBox
              \State $h\gets\mathsf{Recurrent \ Core}(h)$
              \State $r\gets\mathsf{Coda}(h)$
              \State Compute $p_i,\lambda_i$ for $i\in\mathcal U$
                by \cref{eq:hazard}
                \label{alg:decode-heads}
              \State $a_i\gets a_i+(1-a_i)\lambda_i$,
                $i\in\mathcal U$
                \label{alg:decode-cumulative}
              \State $\mathcal B\gets
                \{i\in\mathcal U:\mathsf{Entropy}(p_i)\leq\tau\}$
                \label{alg:decode-commit}
              \State $\widehat y_i\gets\mathsf{Sample}(p_i)$,
                $i\in\mathcal B$
              \State $\mathcal C\gets\mathcal C\cup\mathcal B$;
                $\mathcal U\gets\mathcal U\setminus\mathcal B$
                \label{alg:decode-remove}
              \State $h_i\gets\emb(\widehat y_i)$,
                $i\in\mathcal C$
                \label{alg:decode-feedback}
              \Statex \Comment{Committed tokens become context}
              \State \textbf{if} $\mathcal U=\varnothing$
                \textbf{or} $\operatorname{mean}_{i\in\mathcal U}a_i
                \geq q$ \textbf{then break}
                \label{alg:decode-stop}
            \EndFor
            \If{$\mathcal C=\varnothing$}
              \Comment{Ensure progress}
              \State $j\gets\operatorname*{arg\,min}_{i\in\mathcal U}
                \mathsf{Entropy}(p_i)$
              \State $\widehat y_j\gets\mathsf{Sample}(p_j)$
              \State $\mathcal C\gets\{j\}$
            \EndIf
            \State $x_i\gets\widehat y_i$, $i\in\mathcal C$
              \label{alg:decode-update}
          \EndWhile
          \State \Return $\mathbf x$
        \end{algorithmic}
      \endgroup
    \end{algorithm}
  \end{minipage}
  \vspace{-30pt}
\end{wrapfigure}

\textbf{Token commitment and latent feedback.}
At pass $s$, let $\mathcal B^{(s)}\subseteq\mathcal U$ contain
the unresolved positions whose predictive entropy under $p_i^{(s)}$
is at most $\tau$. For each $i\in\mathcal B^{(s)}$, we sample
$\widehat y_i\sim p_i^{(s)}$ and remove the position from $\mathcal U$ (i.e. committed).
Let $\mathcal C^{(s)}:=\bigcup_{s'=1}^{s}\mathcal B^{(s')}$
denote all positions committed within the current outer step.
The next recurrent input combines their token embeddings with
the core outputs at the remaining positions:
\begin{equation}
    {\color{myred}h_i^{(s)}}=
    \begin{cases}
        \emb(\widehat y_i), & i\in\mathcal C^{(s)},\\
        {\color{lightblue}\widetilde h_i^{(s)}}, & \text{otherwise},
    \end{cases}
    \label{eq:mixed-carry}
\end{equation}
where $\emb$ is the token embedding function.
Thus, committed tokens supply discrete context at subsequent passes,
while unresolved positions build on their accumulated latent states.

\textbf{Adaptive stopping.}
The exit gate controls whether to continue refining the remaining
unresolved positions. Starting from $a_i^{(0)}=0$, we accumulate
the halting probabilities as
$a_i^{(s)}=a_i^{(s-1)}
+(1-a_i^{(s-1)})\lambda_i^{(s)}$.
After removing newly committed positions, the inner loop terminates
if $\mathcal U$ is empty or if the mean cumulative halt probability
over $\mathcal U$ reaches the threshold $q$.
Otherwise, recurrence continues up to depth $K$.

At the end of the inner loop, committed tokens are written into
the sequence. If no token has committed, we select the unresolved
position with the lowest predictive entropy and commit its prediction
to ensure progress. The next outer step reinitializes the recurrent
state from the updated sequence, and decoding continues until no
masked positions remain.

\FloatBarrier
\subsection{Training}
\label{sec:joint-elbo}
\label{sec:policy-gradient}
\label{sec:training}
\label{sec:on-policy-rollout}
\label{sec:policy-optimization}

Training \modelname{} involves jointly learning a denoiser and a halting policy that dictates when to commit predictions. This section first introduces the joint exit schedule (\cref{def:exit-schedule}) and its effect on token predictions. We then derive a variational objective to optimize both components (\cref{eq:conditional-nelbo}). Finally, we present an unbiased gradient estimator to optimize this objective (\cref{thm:one-schedule-estimator}) and introduce
practical regularization and variance reduction techniques.

\textbf{Variational objective.}
To learn token-wise computation allocation, we treat exit depths as latent variables. As each
commitment changes the context for unresolved positions, we model these depths jointly through an
exit schedule.

\begin{definition}[Exit schedule]
\label{def:exit-schedule}
For a corrupted input $\mathbf x_t$, let
$\mathcal M_t:=\{i:x_i=\mathtt{[MASK]}\}$ denote the masked positions.
An \emph{exit schedule} is a discrete random vector
$\mathbf z=(z_i)_{i\in\mathcal M_t}
\in\{1,\ldots,K\}^{|\mathcal M_t|}$, where $z_i=s$ indicates that
token $i$ commits after recurrent pass $s$.
\end{definition}

A noisy sequence $\mathbf{x}_t$ and a fixed exit schedule $\mathbf{z}$ fully specify one outer iteration which comprises up to $K$ inner loops. During this process, early-committed tokens become context for the subsequent inner loops, while the remaining uncommitted tokens continue to refine their latent states. Formally, denote the clean sequence as $\mathbf{y}\in\mathcal{V}^N$. The base model gives us the likelihood of the clean data $\mathbf{y}_{\mathcal{M}_t}$, conditioned on the noisy sequence $\mathbf x_t$ and the exit schedule $\mathbf z$. An intuitive idea is to marginalize the exit schedules and seek to maximize the following conditional likelihood:
\begin{equation}
    p_\theta(\mathbf{y}_{\mathcal{M}_t}\mid\mathbf{x}_t)
    =\sum_{\mathbf{z}}\pi(\mathbf{z})\,
      p_\theta(\mathbf{y}_{\mathcal{M}_t}\mid\mathbf{x}_t,\mathbf{z}).
    \label{eq:schedule-marginal}
\end{equation}
However, this summation ranges over exponentially many schedules, each dictating a separate sequence of model fowards\footnote{For a nonempty masked set,
evaluating a schedule $\mathbf z$ requires $\max_{i\in\mathcal M_t}z_i$ recurrent passes.}:
because committed tokens feed their embeddings back into the recurrent state (\cref{eq:mixed-carry}), altering a single token's exit depth modifies the context for the remaining tokens, thereby shifting the entire subsequent trajectory.

We instead introduce a variational distribution $q_\phi(\mathbf z\mid\mathbf x_t,\mathbf y_{\mathcal M_t})$, parameterized by the exit gate, and derive a negative evidence lower bound (NELBO) with prior distribution $\pi(\mathbf z)$ over exit schedules: 
\begin{equation}
  J(\theta,\phi;\mathbf x_t,\mathbf y)
    :=\E_{\mathbf z\sim q_\phi}\!\Biggl[
      \underbrace{
        -\sum_{i\in\mathcal M_t}
        \log p_{\theta,i}^{(z_i)}(y_i\mid\mathbf x_t;\mathbf z)
      }_{\mathcal L_{\mathrm{traj}}(\mathbf z)}
    \Biggr]
    +\KL(q_\phi\|\pi) \geq -\log p_\theta(\mathbf y_{\mathcal M_t}\mid\mathbf x_t).
    \label{eq:conditional-nelbo}
\end{equation}
The derivation is provided in Appendix~\ref{app:joint-elbo-proof}.
Here, $p_{\theta,i}^{(z_i)}(y_i\mid\mathbf x_t;\mathbf z)$ denotes the predicted likelihood of token $y_i$ at exit depth $z_i$, conditioned on the noisy input $\mathbf x_t$ and the joint exit schedule $\mathbf z$. Summing this loss over all masked positions $i\in\mathcal M_t$ yields the overall trajectory loss $\mathcal L_{\mathrm{traj}}$ for a given schedule. 
To encourage shallow exits, we follow \citet{banino2021pondernet} and formulate the prior distribution $\pi(\mathbf z)$ as a product of geometric distributions truncated to $\{1,\ldots,K\}$.

The NELBO above is evaluated on a single corrupted example $\mathbf x_t$. The full variational training objective is given by:
\begin{equation}
    \boxed{
    \mathcal L_{\mathrm{train}}(\theta,\phi)
    :=\E_{\mathbf y\sim p_{\mathrm{data}},\,t\sim\Unif[0,1],\,\mathbf x_t\sim q_t(\cdot\mid\mathbf y)}\!\left[
        \frac{1}{t}\,
        J(\theta,\phi;\mathbf x_t,\mathbf y)
      \right]
    },
    \label{eq:full-training-objective}
\end{equation}
where $p_{data}$ denotes the data distribution, and $q_t(\cdot \mid \mathbf y)$ denotes the \emph{forward diffusion process} that independently mask each position with probability $t$, and leave it unchanged with probability $1-t$.

\textbf{Gradient estimation.}
To estimate the gradients of \cref{eq:full-training-objective} with respect to $\theta$ and $\phi$, we sample an exit schedule $\mathbf{z} \sim q_\phi$ and evaluate $\mathcal{L}_{\mathrm{traj}}(\mathbf{z})$ along the sampled trajectory. While the denoiser parameters $\theta$ can be updated via standard backpropagation, we cannot directly backpropagate through the exit gate $\phi$ because the discrete $\mathbf{z}$ precludes ordinary pathwise differentiation. To address this, we employ a \emph{score-function estimator}\footnote{Continuous relaxations, such as Gumbel-Softmax~\citep{jang2017categorical}, provide a viable alternative. However, they would require extending \cref{eq:mixed-carry} to accommodate partially committed tokens, potentially introducing additional optimization challenges during training.}
~\citep{williams1992simple}, yielding the following unbiased gradient estimator:

\begin{proposition}[Unbiased gradient estimator]
\label{thm:one-schedule-estimator}
Sample $\mathbf y,t,\mathbf x_t$ as in \cref{eq:full-training-objective}, and sample
$\mathbf z\sim q_\phi$ sequentially during a single recurrent rollout. With the sampled exit
decisions held fixed during backpropagation, define
\begin{equation}
    \widehat{\mathcal L}_{\mathrm{sur}}(\mathbf z)
    =\frac{1}{t}\Biggl\{
    \mathcal L_{\mathrm{traj}}(\mathbf z)
    +
    \stopgrad\left[
        \mathcal L_{\mathrm{traj}}(\mathbf z)
        + \log\frac{q_\phi(\mathbf z)}{\pi(\mathbf z)}
    \right]
    \log q_\phi(\mathbf z)
    \Biggr\}.
    \label{eq:one-schedule-surrogate}
\end{equation}
Under the stop-gradient convention in \cref{eq:hazard}, with the policy features held fixed when
differentiating with respect to $\theta$, this surrogate yields unbiased gradients:
\begin{equation}
    \E_{\mathbf y,t,\mathbf x_t,\mathbf z}
    \!\left[\nabla_{\theta,\phi}\widehat{\mathcal L}_{\mathrm{sur}}(\mathbf z)\right]
    =\nabla_{\theta,\phi}\mathcal L_{\mathrm{train}}(\theta,\phi).
    \label{eq:unbiased-one-schedule-gradient}
\end{equation}
\end{proposition}
The proof is provided in Appendix~\ref{app:policy-gradient-proof}. 
In \cref{eq:one-schedule-surrogate}, the first term $\mathcal{L}_{\mathrm{traj}}$ trains the denoiser to predict target tokens. The second term utilizes the score-function estimator to train the halting policy. 
The detached cost $\mathcal{L}_{\mathrm{traj}}(\mathbf{z})+\log\frac{q_\phi(\mathbf{z})}{\pi(\mathbf{z})}$ provides an outcome signal for the sampled schedule, giving each halting decision credit for its effect on the sequence as a whole.
This encourages the policy to favor schedules that achieve high prediction accuracy while remaining close to the prior distribution.

\textbf{Practical regularization.}
To maintain token-wise adaptivity, we relax the joint KL penalty term $\log\frac{q_\phi(\mathbf{z})}{\pi(\mathbf{z})}$. Specifically, we regularize the average depth distribution across the minibatch toward the geometric prior, while applying a weaker penalty to the variations among individual token depths. We defer the complete formulation to Appendix~\ref{app:unit-weight-geometric-kl}.

\textbf{Variance reduction.}
\cref{eq:one-schedule-surrogate} directly supervises each token at a
single sampled depth, although the same rollout produces predictions at
every preceding pass. We reuse these predictions to give earlier passes
a direct learning signal, even when a token exits later. For token $i$,
we average the per-depth cross-entropy losses as
$\sum_{s=1}^{z_i}\stopgrad[\omega_i^{(s)}]\ell_i^{(s)}$, where
$\omega_i^{(s)}$ is the predicted exit-depth probability after assigning
all mass beyond the sampled exit $z_i$ to $z_i$, so that
$\sum_{s=1}^{z_i}\omega_i^{(s)}=1$. 
This requires no additional recurrent
passes and empirically reduces denoiser gradient variance
(Figure~\ref{fig:ablations}, right). 
Furthermore, we subtract the detached first-pass prediction cost $\ell_i^{(1)}$ as a \emph{control variate}~\citep{williams1992simple}.
The full formulation of the weighting scheme and the baseline derivation is given in Appendix~\ref{app:intermediate-loss-proof}, and the complete training surrogate is
\cref{eq:implemented-training-surrogate}.

\begin{table}[!t]
\centering
\caption{\textbf{Main results.} We compare \modelname{} against its corresponding AR baseline and recent diffusion language models (DLMs). Within each model scale, the overall best results are in \textbf{bold}, and the best non-AR results are
\underline{underlined}. Columns corresponding to our model are shaded in
blue.}
\vspace{-5pt}
\label{tab:main-results}
\begingroup
\setlength{\tabcolsep}{1pt}
\renewcommand{\arraystretch}{1.0}
\small
\newlength{\modelheadwidth}
\newcommand{\modelhead}[2]{{\scriptsize
  \settowidth{\modelheadwidth}{#1-#2}%
  \makebox[\modelheadwidth][c]{\shortstack{#1\\#2}}}}
\newcommand{\modelgroup}[1]{{\scriptsize\shortstack{\textbf{#1
  Baseline}\\\textit{Instruct Model}}}}
\newcommand{\oursgroup}{{\scriptsize\shortstack{\textbf{Ours}\\\textit{Instruct
  Model}}}}
\resizebox{\textwidth}{!}{%
\begin{tabular}{@{}lccO@{\hspace{3pt}{\color{black!35}\vrule width
  0.45pt}\hspace{3pt}}ccccccO@{}}
\toprule
\multirow{2}{*}{\textbf{Benchmark}}
& \multicolumn{1}{c}{\modelgroup{AR}}
& \multicolumn{1}{c}{\modelgroup{DLM}}
& \multicolumn{1}{c}{\oursgroup}
& \multicolumn{1}{c}{\modelgroup{AR}}
& \multicolumn{5}{c}{\modelgroup{DLM}}
& \multicolumn{1}{c}{\oursgroup} \\
\cmidrule(lr){2-2}\cmidrule(lr){3-3}\cmidrule(lr){4-4}
\cmidrule(lr){5-5}\cmidrule(lr){6-10}\cmidrule(lr){11-11}
& \modelhead{Qwen3}{1.7B}
& \modelhead{SDAR}{1.7B}
& \modelhead{\textbf{\modelname{}}}{\textbf{1.7B}}
& \modelhead{Qwen3}{8B}
& \modelhead{LLaDA}{8B}
& \modelhead{Dream}{7B}
& \modelhead{Fast-dLLM-v2}{7B}
& \modelhead{SDAR}{8B}
& \modelhead{WeDLM}{8B}
& \modelhead{\textbf{\modelname{}}}{\textbf{8B}} \\[2pt]
\midrule
\multicolumn{11}{c}{General Reasoning} \\
\midrule
ARC-C
  & \textbf{82.4} & 74.9 & \underline{77.7}
  & 93.9 & 85.6 & 84.3 & 77.2
  & 90.0 & 91.9 & \textbf{\underline{94.4}} \\
ARC-E
  & \textbf{90.0} & 86.5 & \underline{88.7}
  & 96.1 & 92.6 & 93.0 & 83.4
  & 93.4 & 97.5 & \textbf{\underline{98.1}} \\
MMLU
  & 60.3 & \textbf{\underline{63.4}} & 59.4
  & 76.6 & 62.4 & 68.4 & 66.7
  & \textbf{\underline{78.5}} & 78.0 & 76.6 \\
MMLU-Pro
  & 39.0 & 37.0 & \textbf{\underline{39.6}}
  & 56.8 & 35.6 & 42.0 & 40.6
  & 56.3 & 58.3 & \textbf{\underline{63.9}} \\
\midrule
\multicolumn{11}{c}{Math \& Science} \\
\midrule
GSM8K
  & 83.2 & 80.1 & \textbf{\underline{85.5}}
  & 93.6 & 73.8 & 82.0 & 85.1
  & 91.4 & 93.3 & \textbf{\underline{94.2}} \\
MATH-500
  & \textbf{72.0} & \underline{62.4} & 61.8
  & \textbf{81.8} & 42.2 & 42.0 & 58.2
  & 77.0 & 77.8 & \underline{80.8} \\
GPQA-Diamond
  & 31.3 & 32.3 & \textbf{\underline{41.4}}
  & 47.0 & 22.7 & 23.7 & 21.2
  & 38.4 & 37.4 & \textbf{\underline{49.5}} \\
\midrule
\multicolumn{11}{c}{Code Generation} \\
\midrule
MBPP
  & 61.9 & 60.3 & \textbf{\underline{65.5}}
  & 79.0 & 46.0 & 65.5 & 61.9
  & 72.0 & 74.3 & \textbf{\underline{81.5}} \\
MBPP$^+$
  & 58.5 & 59.4 & \textbf{\underline{60.1}}
  & \textbf{74.5} & 43.4 & 60.6 & 50.0
  & 67.9 & 63.5 & \underline{72.5} \\
HumanEval
  & 62.8 & 61.6 & \textbf{\underline{72.6}}
  & 85.4 & 43.3 & 53.7 & 65.9
  & 78.0 & 79.9 & \textbf{\underline{87.8}} \\
HumanEval$^+$
  & 60.4 & 53.0 & \textbf{\underline{67.7}}
  & 79.3 & 38.4 & 50.0 & 61.0
  & 73.2 & 74.4 & \textbf{\underline{84.2}} \\
\midrule
\textbf{Average}
  & 63.8 & 61.0 & \textbf{\underline{65.5}}
  & 78.5 & 53.3 & 60.5 & 61.0
  & 74.2 & 75.1 & \textbf{\underline{80.3}} \\
\bottomrule
\end{tabular}%
}
\endgroup
\vspace{-10pt}
\end{table}

\begin{figure}[!t]
  \centering
  \includegraphics[width=\textwidth]{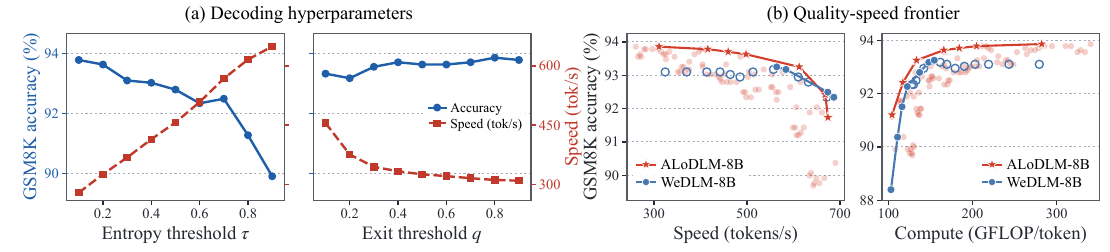}
  \vspace{-15pt}
  \caption{\textbf{(a) Complementary roles of the entropy and halting gates.} Increasing $\tau$ relaxes the commitment criterion and accelerates
  generation, with sharper accuracy losses at high thresholds (left). Increasing $q$ permits deeper latent refinement, also yielding accuracy
  gains at lower throughput (right).
  These gates collectively balance parallel token commitment with adaptive refinement of unresolved tokens.
  \textbf{(b) Adaptive recurrence improves the quality--efficiency frontier.}
  Across most of the evaluated range, \modelname{}-8B achieves higher accuracy
  than WeDLM-8B at comparable throughput (left) or compute per generated token
  (right). Solid lines connect frontier configurations.}
  \vspace{-10pt}
 \label{fig:speed-pareto}
\end{figure}

\begin{figure}[t]
  \centering
  \includegraphics[width=\textwidth]{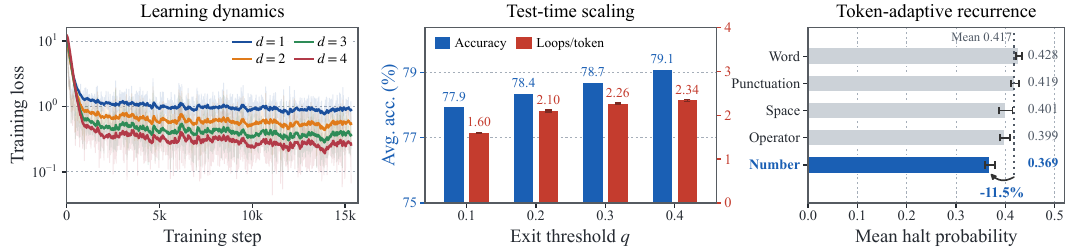}
  \vspace{-15pt}
    \caption{
    \textbf{Left: Learning dynamics.} After the initial training phase,
    deeper recurrent readouts exhibit lower training loss on their active tokens.
    \textbf{Middle: Test-time scaling with $\mathbf{q}$.} Increasing $q$ raises loops-per-token and the average score across eleven benchmarks.
    \textbf{Right: Token-adaptive recurrence.} Mean first-pass halting
    probabilities differ across token categories, with numerical tokens
    showing the lowest mean. These preferences are learned without
    explicit token-difficulty labels or predefined halting targets.}
  \vspace{-15pt}
  \label{fig:training-allocation}
\end{figure}

\section{Experiments}
\label{sec:experiments}

We first describe the experimental setup (\S\ref{sec:evaluation-setup}),
then evaluate generation quality (\S\ref{sec:performance-evaluation}), 
inference speed and computational cost (\S\ref{sec:speed-evaluation}),
followed by analyses and ablations (\S\ref{sec:ablation-studies}).

\subsection{Evaluation Setup}
\label{sec:evaluation-setup}

\textbf{Training details.}
We convert Qwen3-1.7B and Qwen3-8B~\citep{yang2025qwen3} into
\modelname{}-1.7B and \modelname{}-8B, respectively, using WeDLM's
streaming block-diffusion framework~\citep{liu2025wedlm}, which enables
full-sequence conditioning with a standard causal attention mask.
Both models use a maximum recurrent depth of $K=4$.
The recurrent core comprises all 28 Transformer layers at 1.7B
and the middle 16 layers at 8B.
Unlike prior conversion pipelines combining continued pretraining
(CPT) and supervised fine-tuning (SFT)~\citep{tian2025next,liu2025wedlm,cheng2025sdar},
we skip CPT and directly perform SFT on a 5B-token corpus using
AdamW~\citep{loshchilov2017decoupled} with learning rate $10^{-5}$
and weight decay $0.01$.
We defer the full details to Appendix~\ref{app:training-details}.

\textbf{Benchmarks.}
We evaluate general knowledge and scientific reasoning on ARC-Easy and ARC-Challenge~\citep{clark2018arc}, MMLU~\citep{hendrycks2021mmlu}, MMLU-Pro~\citep{wang2024mmlupro}, and GPQA-Diamond~\citep{rein2023gpqa}; mathematical reasoning on GSM8K~\citep{cobbe2021gsm8k} and MATH-500~\citep{hendrycks2021math}; and code generation on HumanEval~\citep{chen2021humaneval}, sanitized MBPP~\citep{austin2021mbpp}, and their EvalPlus variants~\citep{liu2023evalplus}. We report accuracy or exact match for non-code tasks, and execution-based pass@1 for coding benchmarks. Following SDAR~\citep{cheng2025sdar}, we adopt the OpenCompass evaluation protocol~\citep{2023opencompass}, utilizing each model's native chat template, greedy decoding, and a 4,096-token generation limit across all tasks. To ensure a fair comparison in Table~\ref{tab:main-results}, we constrain all methods to generate exactly one token per step according to their respective default decoding rules (\emph{e.g.,} selecting the token with the lowest entropy). Full details are provided in Appendix~\ref{app:evaluation-protocols}.

\textbf{Baselines.}
We compare \modelname{} against both AR baselines and recent diffusion language models. For block diffusion baselines, we mainly focus on models converted from Qwen family of models to ensure fair comparison.
Our main comparison includes Qwen3~\citep{yang2025qwen3},
Dream~\citep{ye2025dream},
LLaDA~\citep{nie2025llada}, Fast-dLLM-v2~\citep{wu2025fastdllmv2}, SDAR~\citep{cheng2025sdar}, and
WeDLM~\citep{liu2025wedlm} at the corresponding parameter scales.  To ensure fair comparison, each model uses its optimized high-speed inference
framework: LLaDA uses dInfer~\citep{ma2025dinfer}, SDAR uses JetEngine, and the WeDLM and Qwen models use
vLLM~\citep{zheng2025vllm}. Our \modelname{} is also served via vLLM using
depth-aware KV caching~\citep{elbayad2019depth,geiping2025recurrent,zhu2025ouro}: for each
prefix token, we use KV states at the same recurrent depth when available,
and otherwise use its deepest available states (see Appendix~\ref{app:depth-kv}
for details).

\subsection{Performance Evaluation}
\label{sec:performance-evaluation}
Table~\ref{tab:main-results} compares \modelname{} against prior methods across eleven benchmarks. The results demonstrate that \modelname{} consistently surpasses all evaluated DLMs, even outperforming the strong AR baselines.
At both the 1.7B and 8B parameter scales, \modelname{} establishes new state-of-the-art averages of 65.5 and 80.3, respectively. \modelname{}-8B outperforms the strong DLM baseline WeDLM-8B by an average of 5.2 points, leading on ten of the eleven tasks. Notably, \modelname{} achieves this performance through direct supervised fine-tuning (SFT) on a modest 5B-token corpus, bypassing the expensive continued-pretraining phase required by SDAR and WeDLM.
Compared to the strong Qwen3 AR models, \modelname{} presents consistent improvements on both 1.7B (63.8 vs. 65.5) and 8B scales (78.5 vs. 80.3). The model shows particularly strong gains in code generation, sweeping all four coding benchmarks over all baselines (except slightly behind the AR models on MBPP+). This advantage may be attributed to its looped architecture, aligning with the findings of \citet{wang2026smelt}.

\begin{figure}[t]
  \centering
  \includegraphics[width=\textwidth]{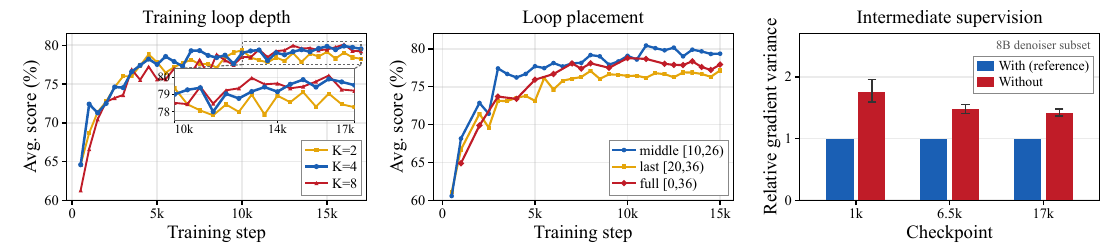}
  \vspace{-15pt}
  \caption{
  \textbf{Left:} At 8B, $K=4$ outperforms $K=8$ early in training
  and achieves comparable performance later, while $K=2$ plateaus
  at a lower score. The inset enlarges steps 10{,}000--17{,}000.
  \textbf{Middle:} Middle-layer recurrence outperforms last-layer
  recurrence later in training. The full-stack curve is a historical
  reference evaluated on a different benchmark subset.
  \textbf{Right:} Intermediate supervision reduces conditional
  gradient variance in \modelname{}-8B. Variance across sampled
  exit trajectories is normalized to the value with intermediate
  supervision at each checkpoint. Error bars show paired 95\%
  bootstrap confidence intervals across eight fixed input packs.
  }
  \label{fig:ablations}
  \vspace{-10pt}
\end{figure}

\subsection{Quality--Efficiency Trade-off}
\label{sec:speed-evaluation}

\textbf{Decoding hyperparameters.}
Figure~\ref{fig:speed-pareto}(a) illustrates how the entropy threshold
$\tau$ and exit threshold $q$ balance parallel token commitment
and latent refinement in \modelname{}-8B on GSM8K.
At fixed $q=0.5$, increasing $\tau$ from 0.1 to 0.6 relaxes the
commitment criterion and raises throughput from 278.7 to 508.3
tokens/s, while accuracy decreases from 93.8\% to 92.3\%.
More aggressive commitment incurs a larger quality penalty,
with accuracy falling to 89.9\% at $\tau=0.9$.
The exit threshold $q$ controls how long unresolved tokens continue
to refine their latent states. At fixed $\tau=0.2$, increasing $q$
from 0.1 to 0.9 raises accuracy from 93.3\% to 93.8\%, while
throughput decreases from 455.3 to 309.0 tokens/s.
Together, these controls allow confident predictions to become
discrete context promptly while reserving further recurrent
computation for unresolved tokens.

\textbf{Quality--efficiency frontier.}
Figure~\ref{fig:speed-pareto}(b) compares \modelname{}-8B with the
strong WeDLM-8B baseline under single-stream inference on one
NVIDIA B200. The evaluation covers 94 $(q,\tau)$ settings for
\modelname{} and 21 entropy-threshold settings for WeDLM.
The accuracy--throughput frontiers cross near 650 tokens/s:
\modelname{} offers better trade-offs in the higher-accuracy regime,
while WeDLM leads at the fastest displayed operating points.
At a matched accuracy of 93.25\%, optimized \modelname{}-8B reaches
612.4 tokens/s versus 564.2 for WeDLM, an 8.5\% throughput gain.

Throughput depends on both arithmetic complexity and execution efficiency (e.g., kernel performance, memory traffic, and runtime overhead). To isolate and evaluate arithmetic efficiency, we compare the estimated GFLOPs per generated token (Figure~\ref{fig:speed-pareto}(b), right). \modelname{} maintains higher accuracy across the entire evaluated compute range. At the 93.25\% accuracy threshold, it requires only 133.5 GFLOPs per token compared to 154.6 for WeDLM, a 13.6\% reduction. These results confirm that adaptive latent recurrence improves generation quality per unit of computation while maintaining competitive parallel decoding speeds.

\subsection{Analysis and Ablations}
\label{sec:ablation-studies}

\textbf{Test-time scaling.}
Figure~\ref{fig:training-allocation} (left) shows that, after the
initial training phase, deeper recurrent readouts in \modelname{}-8B
exhibit lower training loss on their active tokens, consistent
with progressive refinement.
At inference time, increasing the exit threshold $q$ raises the
average number of loops per token from $1.6$ to $2.34$ and the
average score across eleven benchmarks from $77.9\%$ to $79.1\%$
(Figure~\ref{fig:training-allocation}, middle).
These results validate $q$ as a control for average recurrent depth
and demonstrate the effectiveness of latent recurrence as an axis
for test-time scaling.

\textbf{Token-adaptive recurrence.}
Figure~\ref{fig:training-allocation} (right) compares mean first-pass
halt probabilities across token types in \modelname{}-8B, averaged
equally over GSM8K, MATH-500, MBPP-sanitized, and HumanEval.
Numerical tokens have the lowest mean halt probability ($0.369$),
approximately $11.5\%$ below the cross-dataset reference mean
($0.417$), while word tokens have the highest ($0.428$).
This lower halting propensity indicates a stronger learned preference
for further latent refinement at numerical positions.
Such a preference is consistent with the importance of precise
numerical predictions in mathematical reasoning and code generation,
where small errors in values can alter the final result.
These observations support our approach to mitigating the
\textit{computation--difficulty mismatch}: learned token-dependent
halting preferences can guide additional computation toward
predictions requiring further refinement.
This adaptive behavior emerges from jointly learning token prediction
and the halting policy, without explicit difficulty labels or
predefined halting targets for different token categories.

\textbf{Maximum recurrent depth.}
Figure~\ref{fig:ablations} (left) compares \modelname{}-8B trained
with $K\in\{2,4,8\}$ under matched settings, reporting the mean
score across six benchmarks.
At matched training steps, $K=2$ and $K=4$ improve faster initially
than $K=8$, but $K=2$ plateaus at a lower score.
With half the maximum recurrent depth, $K=4$ achieves comparable
performance to $K=8$ later in training, motivating our choice
of $K=4$.

\textbf{Loop placement.}
Prior looped models vary in the extent and placement of
recurrence~\citep{geiping2025recurrent,zhu2025ouro,lee2026loopmdm}.
Figure~\ref{fig:ablations} (middle) compares recurrence over the
middle layers $[10,26)$ and last layers $[20,36)$ of Qwen3-8B,
with the full-stack trajectory $[0,36)$ included as a historical
reference.
Middle-layer recurrence outperforms last-layer recurrence later
in training, although both cores contain 16 layers.
This result suggests that effective latent refinement depends
on where recurrence is applied, motivating the middle-layer core
of \modelname{}-8B.

\textbf{Variance reduction.}
Figure~\ref{fig:ablations} (right) evaluates the effect of intermediate
denoiser supervision on gradient noise in \modelname{}-8B.
Using fixed corrupted inputs and paired exit trajectories, we compare
weighted supervision over pre-commitment readouts with supervision
at the sampled exit alone.
For the measured subset of 4{,}096 denoiser normalization parameters,
removing intermediate supervision increases conditional gradient
variance to $1.76\times$, $1.49\times$, and $1.42\times$ the
full-method value at training steps 1{,}000, 6{,}500, and 17{,}000,
respectively.
These results justify reusing intermediate predictions to reduce
noise from exit-trajectory sampling, providing a less noisy denoiser
training signal for latent refinement without additional recurrent
forward passes.

\FloatBarrier
\section{Limitations}
\label{sec:limitations}
First, \modelname{} can exhibit a longer Time to First Token than a comparable AR model, as the KV-cache is constructed for the full recurrent depth during the prefill phase. This computational overhead can diminish the benefits of parallel decoding for short responses or latency-sensitive applications. Second, generation speed is inherently input-dependent. Because the amount of recurrent computation and the number of tokens committed simultaneously rely on token confidence and adaptive halting decisions, throughput can vary significantly across prompts, datasets, and domains—even under greedy decoding. Consequently, on domains underrepresented in the training data, the model may require additional recurrent refinement and make fewer parallel commitments, potentially reducing or entirely reversing its speed advantage over an optimized AR baseline.

\section{Conclusion}
\label{sec:conclusion}
We introduced \modelname{}, a family of adaptively looped diffusion
language models that address the \textit{computation--difficulty mismatch}
through token-adaptive latent recurrence. Within each denoising step,
confident tokens commit early to provide discrete context, while
unresolved tokens retain and refine their latent states across recurrent
passes. We derived a conditional ELBO over latent token-wise exit
schedules to guide computation allocation through sequence-level credit
assignment, and incorporated intermediate supervision to support
denoiser training. Across eleven benchmarks, \modelname{} outperforms
all evaluated diffusion models and the corresponding Qwen3 baselines
in average score at both 1.7B and 8B scales. With optimized inference,
\modelname{}-8B delivers approximately $2.7\times$ the throughput of
vLLM-served Qwen3-8B at comparable accuracy on GSM8K.

\subsection*{AI use statement}
We used generative AI tools to assist with developing and checking
mathematical derivations and proofs, implementing methods, revising
the manuscript, and creating and modifying scientific figures.
We reviewed all AI-assisted material, verified the mathematical
arguments and implementation, and take responsibility for the
final text, claims, and artifacts.

\begingroup
\let\originalthebibliography\thebibliography
\renewcommand{\thebibliography}[1]{%
  \originalthebibliography{#1}\fussy\setlength{\emergencystretch}{2em}}
\bibliography{iclr2027_conference}
\bibliographystyle{iclr2027_conference}
\endgroup

\clearpage
\appendix
\addtocontents{toc}{\protect\setcounter{tocdepth}{2}}

\begingroup
  \renewcommand{\contentsname}{Appendix Contents}
  \hypersetup{linktoc=all}
  \pdfbookmark[1]{Appendix Contents}{appendix-contents}
  \tableofcontents
\endgroup

\newpage
\section{Additional Related Work}
\label{sec:related-work}

\textbf{Discrete Diffusion Language Models} generally fall into two families based on their corruption processes: masked DLMs and uniform DLMs~\citep{austin2021structured}. During the reverse process, standard masked diffusion fixes tokens once unmasked, whereas uniform diffusion allows for continuous revision across subsequent denoising steps. While most established DLM rely on masked diffusion~\citep{sahoo2024mdlm,nie2025llada,ye2025dream,song2025seed,arriola2025block,cheng2025sdar,wu2025fastdllmv2,liu2025wedlm,lee2026loopmdm,carballocastro2026recursive,hersche2026softmasked,miele2026fixedpoint}, recent successes like DiffusionGemma and Duo highlight the growing viability of uniform DLMs~\citep{team2026diffusiongemma,sahoo2025duo}. However, despite their parallel generation capabilities, masked DLMs continue to face a quality gap when compared to strong autoregressive (AR) models~\citep{arriola2025block,Liu2025TiDARTI}.
To bridge this gap, recent work mainly focuses on two trajectories. The first line of research reintroduces AR structure into the diffusion process. For example, Block Diffusion generates sequences in autoregressive blocks while denoising tokens within each block in parallel~\citep{arriola2025block}—a blockwise formulation that Fast-dLLM v2~\citep{wu2025fastdllmv2} and SDAR~\citep{cheng2025sdar} extend to pretrained AR models. Other approaches integrate AR and diffusion more deeply: WeDLM~\citep{liu2025wedlm} employs topological reordering to merge parallel mask recovery with standard causal attention and streaming prefix commitment, while Nemotron-Labs-Diffusion~\citep{Fu2026NemotronLabsDiffusionAT} jointly trains on AR and block-diffusion objectives to support standard generation, diffusion, and self-speculative decoding within a single architecture.
A second line of research exploits diffusion's parallelism to draft candidates for subsequent AR verification. In single-pass architectures, TiDAR~\citep{Liu2025TiDARTI} executes diffusion drafting and AR verification simultaneously, while I-DLM~\citep{yu2026introspective} uses introspective consistency training and strided decoding to verify past proposals while generating new ones concurrently. Other methods decouple the drafting process: DFlash~\citep{Chen2026DFlashBD} employs a separate, lightweight block-diffusion drafter conditioned on the target AR model's context features, whereas Uno~\citep{sahoo2026unlocking} learns parallel diffusion adapters atop frozen AR weights, utilizing rejection sampling against the AR distribution to achieve lossless acceleration.

\textbf{Looped Transformers} reuse shared architectural blocks to increase effective depth without expanding their parameter count, enabling strong algorithmic computation and parameter-efficient in-context learning~\citep{giannou2023looped,yang2024looped,saunshi2025latent}. Recent architectural advances have enriched this paradigm by integrating recurrence with mixture-of-experts~\citep{xue2022go,gao2026loop}, efficient caching mechanisms~\citep{sun2026universal}, agentic post-training~\citep{lab2026nanbeige4}, and variable-depth latent reasoning~\citep{geiping2025recurrent,jeddi2026loopformer}. These models also exhibit unique scaling properties: virtual-depth scaling enhances reasoning capabilities with minimal impact on knowledge capacity~\citep{zhu2025beyond}, while SMELT demonstrates significant compute-efficiency gains for looped MoE architectures under matched resource budgets~\citep{wang2026smelt}. To further optimize these models, Parcae stabilizes recurrence and establishes training scaling laws that jointly optimize loops and data, alongside a saturating inference scaling law~\citep{prairie2026parcae}. Additionally, \citet{chen2026scalingexponents} demonstrate that depth growth and boundary operators can improve scaling exponents, with static looping serving as an effective regularizer for multi-epoch training on limited data.
To support adaptive computation within these loops, models such as ACT and Universal Transformers learn halting mechanisms using weighted intermediate states or outputs~\citep{graves2016adaptive,dehghani2019universal}, whereas Mixture-of-Recursions dynamically routes tokens to varying recurrent depths~\citep{bae2025mixture}. Other approaches focus on stabilizing deep recurrence through fixed-point reasoners~\citep{movahedi2026fixed} or by learning solution attractors~\citep{huang2026equilibrium}. To train these adaptive halting mechanisms, frameworks like PonderNet and Ouro explicitly sum prediction losses across candidate exit depths using learned halting probabilities~\citep{banino2021pondernet,zhu2025ouro}. However, this approach presents a critical challenge for \modelname{}: because early token commitments inherently alter the context for subsequent, unresolved positions, computing an exact summation over all joint commitment schedules becomes computationally intractable.

\textbf{Self-conditioning} is a technique that improves diffusion sample quality by feeding the denoiser its own previous clean-data estimate as an auxiliary input. Originally introduced in Analog Bits~\citep{chen2023analogbits}, the standard training recipe constructs this input via an additional forward pass and applies a stop-gradient to the resulting estimate. During sampling, the estimate from the preceding denoising step is reused, creating a recursive feedback loop that indirectly propagates information from earlier steps. Subsequent work has adapted this paradigm in various ways. Simple Self-Conditioning~\citep{cardei2026selfconditioning} applies detached prediction feedback directly to masked diffusion, whereas DiffusionGemma~\citep{team2026diffusiongemma} maps predicted probabilities into continuous conditioning features. Other related approaches modify the feedback mechanism itself: Soft-Masked DLMs~\citep{hersche2026softmasked} blend mask embeddings with confidence-weighted top-$k$ token embeddings, while DMax~\citep{chen2026dmax} pairs soft embedding feedback with on-policy training. State reuse has also been extended beyond pure predictions. For instance, Loopholing~\citep{jo2026loopholing} propagates hidden states across denoising steps using a two-pass self-conditioning training scheme, and CoFRe~\citep{miele2026fixedpoint} reuses latent states to initialize successive fixed-point solvers.In contrast to conventional self-conditioning methods that rely on detached prediction feedback across denoising steps, \modelname{} maintains persistent hidden states across recurrent passes within each denoising step. By backpropagating through these recurrent passes, \modelname{} explicitly trains earlier states to support future predictions, utilizing a conditional NELBO to jointly learn both the denoiser and a token-wise halting policy.

\section{Additional Method Details}
\label{app:method-details}

This section describes depth-aware KV caching and the computational cost
of recurrent inference.

\subsection{Depth-aware KV caching}
\label{app:depth-kv}

Weight sharing does not imply that attention keys and values are
identical across recurrent passes. The same physical layer receives
different hidden states at different depths. We therefore distinguish
a layer index $\ell$ from its recurrent invocation $d$ and maintain
depth-specific logical buffers for the recurrent core:
\begin{equation}
    \{K_\ell^{(d)},V_\ell^{(d)}\},
    \qquad
    \ell\in[L_{\mathrm{start}},L_{\mathrm{end}}),
    \quad d\in\{1,\ldots,K\}.
    \label{eq:depth-kv}
\end{equation}
Prefill supplies prompt context at the required depths. During
decoding, a depth-$d$ attention invocation uses prefix entries
computed at depth $d$ when available, using the original sequence positions and the
base model's attention mask.

\textbf{Available depth and commitment depth.}
Let $D_{\ell,i}$ be the deepest computed cache depth for prefix
position $i$ at layer $\ell$. When available depths form the prefix
$1,\ldots,D_{\ell,i}$, the matching-depth lookup with
deepest-available fallback reads
\begin{equation}
    \bigl(K_{\ell,i}^{(\min(d,D_{\ell,i}))},
          V_{\ell,i}^{(\min(d,D_{\ell,i}))}\bigr).
    \label{eq:kv-cap}
\end{equation}
The available cache depth need not equal the token's commitment
depth. In a dense recurrent pass, an already committed position can
still be processed as embedding context for unresolved positions;
the depth-specific entries produced by that computation are valid
and should be retained. Reusing an earlier entry is necessary only
when the matching entry is absent. In particular, embedding feedback
should not be replaced by freezing every committed position's
pre-commitment latent cache.

\textbf{Coda readouts.}
For middle-layer recurrence, every pass branches the recurrent carry
through the coda to obtain its prediction and gate readout. The
next recurrent pass continues from the pre-coda carry. The serving
implementation allocates separate depth buffers to recurrent layers,
while prelude and coda layers use one buffer per physical layer.
Coda entries for the active window are refreshed by each readout
branch, so cached coda context uses the latest evaluated branch.
This storage convention is distinct from allocating a separate
cache for every coda invocation.

\textbf{Memory and interpretation.}
For $L$ total layers and $m=L_{\mathrm{end}}-L_{\mathrm{start}}$
recurrent layers, this layout has $L+(K-1)m$ logical KV-buffer
pairs. At 8B, $L=36$, $m=16$, and $K=4$ give 84 pairs, of which
64 belong to the recurrent core. The fallback and shared coda
buffers are serving approximations: their outputs need not match
those obtained by recomputing every prefix token at every depth. Related
early-exit and recurrent models face the same need to provide
attention context when computation depths differ
\citep{elbayad2019depth,geiping2025recurrent,zhu2025ouro}.

\subsection{Accounting for recurrent computation}
\label{app:compute-accounting}

Layer evaluations clarify the cost of a middle-layer recurrent core.
With $S$ executed passes, the prelude runs once, whereas both the
core and the coda readout run at every pass:
\begin{equation}
    C_{\mathrm{layers}}(S)
      =L_{\mathrm{start}}
       +S(L_{\mathrm{end}}-L_{\mathrm{start}})
       +S(L-L_{\mathrm{end}}).
    \label{eq:middle-loop-cost}
\end{equation}
For the 8B partition $[10,26)$, the full four-pass path evaluates
$10+4\cdot16+4\cdot10=114$ Transformer layers, compared with
$4\cdot36=144$ for four full-stack passes. The 1.7B configuration
repeats all 28 layers, giving $28S$ evaluations.

These are per-position layer counts for a dense execution path, not
FLOP counts per generated token. A window with unresolved positions can be processed
several times, and a pass can commit several tokens. Moreover,
positions already committed within the window can still incur
computation as context. With batching, the shared execution can run
to the largest stopping depth among the active sequences. We
therefore measure serving cost from executed operations and
generated-token counts, rather than inferring speed from the
distribution of token commitment depths alone.

\section{Training Objective and Gradient Derivations}
\label{app:derivations}

We give the variational-bound and gradient-estimator proofs, followed by
the relaxed depth regularizer and variance-reduction details deferred
from Section~\ref{sec:training}.

\subsection{Schedule-conditioned likelihood and variational bound}
\label{app:joint-elbo-proof}

Fix a corrupted example $\mathbf x_t$ with $t>0$ and masked positions
$\mathcal M_t$. For a schedule $\mathbf z$, let
$B_s(\mathbf z)=\{i\in\mathcal M_t:z_i=s\}$ and
$G_{<s}(\mathbf z)=\bigcup_{r<s}B_r(\mathbf z)$.
The readouts at pass $s$ are computed before inserting the targets
of tokens in $B_s$. They may depend on earlier committed targets
$\mathbf y_{G_{<s}}$, which are fed back during likelihood evaluation,
but not on the targets of tokens that remain active. Thus
\begin{equation}
\begin{aligned}
    p_\theta(\mathbf y_{\mathcal M_t}\mid\mathbf x_t,\mathbf z)
    &=\prod_{s=1}^{K}\prod_{i\in B_s(\mathbf z)}
       p_{\theta,i}^{(s)}\!\left(
          y_i\mid\mathbf x_t,\mathbf y_{G_{<s}(\mathbf z)};\mathbf z
       \right)\\
    &=\exp[-\mathcal L_{\mathrm{traj}}(\mathbf z)].
\end{aligned}
    \label{eq:schedule-conditioned-likelihood}
\end{equation}
Each readout is a normalized categorical distribution, and tokens
within an exit group are conditionally independent given earlier
groups under the parallel readout model. Summing over the last group's
targets gives one; repeating in reverse exit order proves that
\cref{eq:schedule-conditioned-likelihood} is normalized.
Using observed earlier targets is teacher forcing for this exit order;
generation from the same conditional model instead feeds back samples.

Let $\pi(\mathbf z)$ be a normalized prior independent of the unknown
masked targets. It may depend on the known mask set. Then
\cref{eq:schedule-marginal} defines a normalized conditional likelihood.
The variational distribution
$q_\phi(\mathbf z\mid\mathbf x_t,\mathbf y_{\mathcal M_t})$ may depend
on the observed targets; this does not change the generative model.

\begin{proposition}[Conditional negative ELBO]
\label{thm:latent-schedule-elbo}
For any $q_\phi$ supported on schedules with positive prior probability,
\begin{equation}
    -\log p_\theta(\mathbf y_{\mathcal M_t}\mid\mathbf x_t)
    \leq J
    :=\E_{q_\phi}[\mathcal L_{\mathrm{traj}}(\mathbf z)]
      +\KL(q_\phi\|\pi).
    \label{eq:joint-elbo}
\end{equation}
When these quantities are finite, the gap is
\begin{equation}
    J+\log p_\theta(\mathbf y_{\mathcal M_t}\mid\mathbf x_t)
    =\KL\!\left(
        q_\phi(\mathbf z\mid\mathbf x_t,\mathbf y_{\mathcal M_t})
        \,\|\,
        p_\theta(\mathbf z\mid\mathbf x_t,\mathbf y_{\mathcal M_t})
      \right).
    \label{eq:posterior-gap}
\end{equation}
The bound is tight exactly when $q_\phi$ equals the schedule posterior.
\end{proposition}

\begin{proof}
Write $S_q$ for the support of $q_\phi$ and suppress conditioning
arguments. Restricting the marginal sum to $S_q$ and applying
Jensen's inequality gives
\begin{align}
    \log p_\theta(\mathbf y_{\mathcal M_t}\mid\mathbf x_t)
    &\geq\log\sum_{\mathbf z\in S_q}q_\phi(\mathbf z)
       \frac{\pi(\mathbf z)
       p_\theta(\mathbf y_{\mathcal M_t}\mid\mathbf x_t,\mathbf z)}
       {q_\phi(\mathbf z)}\nonumber\\
    &\geq\E_{q_\phi}\!\left[
        -\mathcal L_{\mathrm{traj}}(\mathbf z)
        +\log\pi(\mathbf z)-\log q_\phi(\mathbf z)
      \right]=-J.
    \label{eq:conditional-jensen-proof}
\end{align}
For the gap, Bayes' rule yields
\[
    \log\frac{q_\phi(\mathbf z)}
        {p_\theta(\mathbf z\mid\mathbf x_t,\mathbf y_{\mathcal M_t})}
    =\log\frac{q_\phi(\mathbf z)}{\pi(\mathbf z)}
      +\mathcal L_{\mathrm{traj}}(\mathbf z)
      +\log p_\theta(\mathbf y_{\mathcal M_t}\mid\mathbf x_t).
\]
Taking the $q_\phi$-expectation proves \cref{eq:posterior-gap};
nonnegativity of KL gives the equality condition.
\end{proof}

This result concerns the joint KL in \cref{eq:conditional-nelbo}.
Appendix~\ref{app:unit-weight-geometric-kl} describes its relaxation
and the coefficients that recover the exact KL.

\subsection{Single-trajectory gradient estimator}
\label{app:policy-gradient-proof}

\textbf{Schedule probability.}
Let $\mathcal F_s$ denote the history immediately before the exit
decisions at pass $s$, including preceding commitments and current
readouts. Active tokens exit independently conditional on
$\mathcal F_s$, with probabilities $\lambda_i^{(s)}$; any remaining
tokens exit deterministically at $K$. Define token $i$'s accumulated
decision log-probability by
\begin{equation}
\begin{aligned}
    u_i(\mathbf z)
    &=\sum_{s<z_i}\log(1-\lambda_i^{(s)})
       +\mathbf 1\{z_i<K\}\log\lambda_i^{(z_i)},\\
    \log q_\phi(\mathbf z\mid\mathbf x_t,\mathbf y_{\mathcal M_t})
    &=\sum_{i\in\mathcal M_t}u_i(\mathbf z).
\end{aligned}
    \label{eq:schedule-logprob}
\end{equation}
The hazards are evaluated along the realized history. Earlier exits
can change later hazards for other tokens, so this chain-rule identity
does not assume independent exit depths or identify $u_i$ with a
marginal log-probability. Normalized decisions and forced final exit
make $q_\phi$ a normalized distribution over the finite decision tree.
We assume differentiable probabilities on a fixed support, as obtained
with finite sigmoid logits.

\textbf{Derivative convention.}
To formalize the detached gate inputs in \cref{eq:hazard}, introduce
a reference backbone $\bar\theta$ for the policy and write
\begin{equation}
    J_{\bar\theta}(\theta,\phi)
    =\E_{q_{\bar\theta,\phi}}\!\left[
        \mathcal L_{\mathrm{traj},\theta}(\mathbf z)
        +\log\frac{q_{\bar\theta,\phi}(\mathbf z)}{\pi(\mathbf z)}
      \right].
    \label{eq:fixed-policy-objective}
\end{equation}
The forward pass uses $\bar\theta=\theta$, but the denoiser derivative
holds $\bar\theta$ fixed. Thus $\nabla_\theta J$ below denotes
$\partial_\theta J_{\bar\theta}|_{\bar\theta=\theta}$, rather than a
total derivative through the policy features. This notation requires
no second backbone. At fixed decisions, the trajectory loss depends
on $\theta$ but not on the separate gate parameters $\phi$; the prior
is fixed with respect to both parameter groups.

\begin{proof}[Proof of Proposition~\ref{thm:one-schedule-estimator}]
Fix $(\mathbf y,t,\mathbf x_t)$ and suppress the reference backbone
and conditioning arguments. Under the stated derivative convention,
\begin{equation}
    \nabla_\theta J
    =\sum_{\mathbf z}q_\phi(\mathbf z)
        \nabla_\theta\mathcal L_{\mathrm{traj}}(\mathbf z)
    =\E_{q_\phi}[\nabla_\theta\mathcal L_{\mathrm{traj}}].
    \label{eq:fixed-policy-denoiser-gradient}
\end{equation}
For the gate parameters, differentiation gives
\begin{align}
    \nabla_\phi J
    &=\sum_{\mathbf z}\nabla_\phi q_\phi(\mathbf z)
       \left[\mathcal L_{\mathrm{traj}}(\mathbf z)
         +\log\frac{q_\phi(\mathbf z)}{\pi(\mathbf z)}\right]
       +\sum_{\mathbf z}q_\phi(\mathbf z)
          \nabla_\phi\log q_\phi(\mathbf z)\nonumber\\
    &=\E_{q_\phi}\!\left[
        \left(\mathcal L_{\mathrm{traj}}+\log\frac{q_\phi}{\pi}\right)
        \nabla_\phi\log q_\phi
      \right],
    \label{eq:single-trajectory-gate-gradient}
\end{align}
where the last sum vanishes because
\begin{equation}
    \E_{q_\phi}[\nabla_\phi\log q_\phi]
    =\sum_{\mathbf z}\nabla_\phi q_\phi(\mathbf z)
    =\nabla_\phi 1=0.
    \label{eq:zero-mean-score}
\end{equation}
With the sampled decisions fixed and the cost detached,
\cref{eq:one-schedule-surrogate} has derivatives
\begin{equation}
\begin{aligned}
    \nabla_\theta\widehat{\mathcal L}_{\mathrm{sur}}
      &=\frac1t\nabla_\theta\mathcal L_{\mathrm{traj}},\\
    \nabla_\phi\widehat{\mathcal L}_{\mathrm{sur}}
      &=\frac1t\left(
          \mathcal L_{\mathrm{traj}}+\log\frac{q_\phi}{\pi}
        \right)\nabla_\phi\log q_\phi.
\end{aligned}
    \label{eq:single-surrogate-derivatives}
\end{equation}
Their schedule expectations are $(1/t)\nabla_{\theta,\phi}J$.
Let $\mu$ denote the parameter-independent law of the clean data,
corruption level, and corrupted input in \cref{eq:full-training-objective}.
Assuming the $1/t$-weighted gradients are integrable and differentiation
may pass through the outer expectation, the tower property gives
\begin{equation}
    \E_{\mu,q_\phi}
      [\nabla_{\theta,\phi}\widehat{\mathcal L}_{\mathrm{sur}}]
    =\E_\mu\!\left[\frac1t\nabla_{\theta,\phi}J\right]
    =\nabla_{\theta,\phi}\mathcal L_{\mathrm{train}}.
    \label{eq:data-averaged-gradient-proof}
\end{equation}
The denoiser derivative throughout retains the fixed-policy meaning
defined above.
\end{proof}

\subsection{Geometric prior and depth regularization}
\label{app:unit-weight-geometric-kl}

The geometric prior in \cref{eq:conditional-nelbo} is truncated and
renormalized over $\{1,\ldots,K\}$:
\begin{equation}
    \pi_{0,c}(d)=\frac{e^{-cd}}{\sum_{r=1}^{K}e^{-cr}},
    \qquad
    \pi_c(\mathbf z)=\prod_{i\in\mathcal M_t}\pi_{0,c}(z_i).
    \label{eq:geometric-product-prior}
\end{equation}
Here, $c>0$ controls the preference for shallow exits. We use $K=4$ and $c=0.4$, giving
$\pi_{0,c}\approx(0.4131,0.2769,0.1856,0.1244)$ and a prior mean depth
$\E_{D\sim\pi_{0,c}}[D]\approx2.02$.
Larger $c$ favors shallower depths; the learned mean depth also depends
on the prediction objective.

\textbf{Separating individual and aggregate regularization.}
For $M>0$ supervised masked positions with fixed positive depth
distributions $r_i$, write
$\bar r=M^{-1}\sum_i r_i$. Adding and subtracting $\log\bar r$ gives
\begin{equation}
\begin{aligned}
    R_{\alpha,\beta}
    &=\frac{\alpha}{M}\sum_i\KL(r_i\|\pi_{0,c})
       +(\beta-\alpha)\KL(\bar r\|\pi_{0,c})\\
    &=\frac{\alpha}{M}\sum_i\KL(r_i\|\bar r)
       +\beta\KL(\bar r\|\pi_{0,c}).
\end{aligned}
    \label{eq:depth-profile-regularizer}
\end{equation}
The first term in the second line penalizes differences among token
profiles; the second matches their average to the prior. We use
$(\alpha,\beta)=(0.1,1)$ to allow tokens to favor different depths
while retaining stronger aggregate prior matching. For samples with
marginals $z_i\sim r_i$, this objective equals
$\E[M^{-1}\sum_i\widetilde C_i]$, where
\begin{equation}
    \widetilde C_i
      =\alpha\log\frac{r_i(z_i)}{\bar r(z_i)}
       +\beta\log\frac{\bar r(z_i)}{\pi_{0,c}(z_i)}.
    \label{eq:ideal-depth-regularization-sample}
\end{equation}

\textbf{The sampled regularization cost.}
Let $\mathbf Z=(\mathbf z_b)_b$ collect the sampled schedules in a
local minibatch, and let $M$ count its supervised masked positions.
Using gate outputs along these rollouts, form
\begin{equation}
    \widehat r_i(d)=
    \begin{cases}
       \lambda_i^{(d)}\prod_{s<d}(1-\lambda_i^{(s)}), & d<K,\\
       \prod_{s<K}(1-\lambda_i^{(s)}), & d=K,
    \end{cases}
    \qquad
    b_{\mathbf Z}(d)=\frac1M\sum_i\widehat r_i(d).
    \label{eq:rollout-aggregate-profile}
\end{equation}
These full-depth profiles require gate readouts at all nonterminal depths
$s<K$, including after commitment. They summarize the realized rollouts rather
than exact marginal probabilities over alternative histories.
Replacing $\log r_i(z_i)$
by the realized decision score $u_i$ in \cref{eq:schedule-logprob}
and $\bar r$ by $b_{\mathbf Z}$ gives
\begin{equation}
    C_i^{\alpha,\beta}(\mathbf Z)
      =\alpha[u_i-\log b_{\mathbf Z}(z_i)]
       +\beta[\log b_{\mathbf Z}(z_i)-\log\pi_{0,c}(z_i)].
    \label{eq:rollout-depth-regularization-cost}
\end{equation}
The cost is detached inside the sequence advantage. Because the
profiles depend on the sampled history, its expectation need not
equal $R_{\alpha,\beta}$; it is a rollout-based relaxation of that design.

\begin{theorem}[Unit-weight equivalence]
\label{thm:unit-weight-geometric-kl}
Fix a corrupted minibatch. For any sequence $b$ with supervised
positions $\mathcal M_b$, $\alpha=\beta=1$, and any positive aggregate profile,
\begin{equation}
    \E_{\mathbf Z}\!\left[
       \sum_{i\in\mathcal M_b} C_i^{1,1}(\mathbf Z)\right]
      =\KL\!\left(q_\phi(\mathbf z_b)\|\pi_c(\mathbf z_b)\right).
    \label{eq:unit-weight-joint-kl}
\end{equation}
Independence of the exit depths within the sequence is not required.
\end{theorem}

\begin{proof}
The aggregate terms cancel for every realized minibatch. By
\cref{eq:schedule-logprob} and the product prior,
\[
    \sum_{i\in\mathcal M_b} C_i^{1,1}(\mathbf Z)
      =\sum_{i\in\mathcal M_b}[u_i-\log\pi_{0,c}(z_i)]
      =\log\frac{q_\phi(\mathbf z_b)}{\pi_c(\mathbf z_b)}.
\]
Taking the expectation proves the claim.
\end{proof}

Thus unit coefficients recover the original joint KL in expectation,
whereas $(0.1,1)$ gives a rollout-based relaxation with weaker
regularization of individual variation.
In the first line of \cref{eq:depth-profile-regularizer}, these settings
correspond to KL weights $(1,0)$ and $(0.1,0.9)$, respectively.

\subsection{Variance reduction and the training surrogate}
\label{app:intermediate-loss-proof}
\label{app:practical-supervision}

\textbf{First-pass control variate.}
For each independently attended sequence $b$, let $\mathcal M_b$
contain its supervised masked positions and let
$M=\sum_b|\mathcal M_b|>0$. Let $\ell_i^{(s)}$ denote the cross-entropy
loss for token $i$ at pass $s$. Define the sequence cost and baseline as
\begin{equation}
    R_b=\sum_{i\in\mathcal M_b}
        [\ell_i^{(z_i)}+C_i^{\alpha,\beta}(\mathbf Z)],
    \qquad
    B_b=\sum_{i\in\mathcal M_b}\ell_i^{(1)}.
    \label{eq:implemented-outcome-return}
\end{equation}
The first-pass readout precedes every sampled commitment. Conditional
on the corrupted inputs and current parameters, $B_b$ is therefore
independent of the sampled schedule, and
\begin{equation}
    \E_{q_\phi}[B_b\nabla_\phi\log q_\phi(\mathbf z_b)]
      =B_b\nabla_\phi\sum_{\mathbf z_b}q_\phi(\mathbf z_b)=0.
    \label{eq:first-pass-outcome-baseline}
\end{equation}
Subtracting this detached baseline leaves the expected gate update
unchanged. The resulting gate surrogate (actor) is
\begin{equation}
    \widehat{\mathcal L}_{\mathrm{gate}}
      =\frac1M\sum_b\stopgrad[R_b-B_b]\log q_\phi(\mathbf z_b).
    \label{eq:implemented-outcome-actor}
\end{equation}
Returns and decision scores are summed within each independently
attended sequence, including separate segments of packed examples.
There is no current-minibatch advantage centering. The baseline supplies
an input-dependent reference without an extra forward pass; its variance
benefit depends on how well it tracks the sampled sequence cost.

\textbf{Intermediate denoiser supervision.}
For a token exiting at $z_i$, let
$v_i^{(s)}=\prod_{r<s}(1-\lambda_i^{(r)})$ and define
\begin{equation}
    \omega_i^{(s)}=
    \begin{cases}
       v_i^{(s)}\lambda_i^{(s)}, & s<z_i,\\
       v_i^{(s)}, & s=z_i,\\
       0, & s>z_i.
    \end{cases}
    \qquad \sum_{s\leq z_i}\omega_i^{(s)}=1.
    \label{eq:implemented-intermediate-weights}
\end{equation}
These weights assign the remaining probability mass to the sampled
exit and exclude readouts after target feedback. For the single-$t$
corruption used in the theoretical objective, detaching these weights
gives the per-example denoiser loss
\begin{equation}
    \widehat{\mathcal L}_{\mathrm{den}}
      =\frac1t\sum_{i\in\mathcal M_t}\sum_{s\leq z_i}
          \stopgrad[\omega_i^{(s)}]\ell_i^{(s)}.
    \label{eq:implemented-denoiser-loss}
\end{equation}
Under blockwise corruption, each token instead uses its block's masking
probability. Minibatch normalization follows the implementation described
in Appendix~\ref{app:experimental-training}.
This reuses pre-commitment predictions to train the denoiser; the gate cost in
\cref{eq:implemented-outcome-return}
still uses each token's sampled-exit loss.
The weighted denoiser objective need not have the same expected
gradient as sampled-exit supervision. We assess its variance empirically:
Figure~\ref{fig:ablations} (right) compares the trace of the gradient
covariance for the measured parameter subset, using fixed corrupted
inputs and paired sampled trajectories. This measures variation from
trajectory sampling rather than from sampling training examples.

\textbf{Complete training surrogate.}
After aggregating and normalizing the denoiser loss as described above, we combine it with the gate surrogate and an auxiliary AR loss $\widehat{\mathcal L}_{\mathrm{AR}}$. This auxiliary loss is computed on the same sequences using standard next-token prediction, which has been shown to be beneficial for AR-to-DLM conversion~\citep{liu2025wedlm,Fu2026NemotronLabsDiffusionAT,tian2025next,Gat2025SetBD}. With equal diffusion/AR weighting, the complete surrogate is:
\begin{equation}
    \widehat{\mathcal L}_{\mathrm{impl}}
      =
          \widehat{\mathcal L}_{\mathrm{den}}
          +\widehat{\mathcal L}_{\mathrm{gate}}
          +\widehat{\mathcal L}_{\mathrm{AR}}.
    \label{eq:implemented-training-surrogate}
\end{equation}
The actor uses sequence-summed costs with a single $1/M$ normalization;
inverse-mask-probability weighting applies only to the denoiser. This surrogate
combines the sequence-level credit structure of
Proposition~\ref{thm:one-schedule-estimator} with relaxed depth
regularization and intermediate supervision. The proposition's
unbiasedness result concerns the original joint-KL objective.

\section{Experimental Details}
\label{app:training-details}
\label{sec:training-details}

This appendix describes the model configurations, optimization settings, and
evaluation protocols used in Section~\ref{sec:experiments}. We distinguish the
single-token decoding protocol used to compare generation quality from the
parallel decoding protocol used to measure the quality--efficiency tradeoff.

\subsection{Models and training}
\label{app:experimental-training}

\textbf{Architecture.}
We initialize \modelname{}-1.7B and \modelname{}-8B from the corresponding
Qwen3 backbones~\citep{yang2025qwen3}. The 1.7B model applies its complete
28-layer Transformer stack recurrently. The 8B model partitions its
36 layers into a 10-layer prelude, a shared 16-layer recurrent core, and a
10-layer coda. Layer ranges below use zero-based, half-open indexing.
Both configurations use a maximum recurrent depth of four and untied input
and output embeddings.
The coda and output head produce a readout at each
recurrent pass; the coda is therefore re-evaluated when the core recurs.

\begin{table}[ht]
\centering
\caption{Model configurations and common optimization settings. A batch
contains 256 packed examples across all workers and accumulation steps.}
\label{tab:training-configurations}
\setlength{\tabcolsep}{6pt}
\begin{tabular}{@{}lcc@{}}
\toprule
Setting & \modelname{}-1.7B & \modelname{}-8B \\
\midrule
Transformer layers & 28 & 36 \\
Prelude layers & None & $[0,10)$ \\
Recurrent core & $[0,28)$ & $[10,26)$ \\
Coda layers & None & $[26,36)$ \\
Maximum recurrent depth & 4 & 4 \\
Packed sequence length & 4,096 & 4,096 \\
Training block size & 32 & 32 \\
Global batch size & 256 & 256 \\
Optimizer & AdamW & AdamW \\
Peak learning rate & $10^{-5}$ & $10^{-5}$ \\
Learning-rate schedule & Cosine & Cosine \\
Warmup fraction & $0.03$ & $0.03$ \\
Weight decay & $0.01$ & $0.01$ \\
Gradient clipping norm & $1.0$ & $1.0$ \\
Arithmetic & bfloat16 & bfloat16 \\
\bottomrule
\end{tabular}
\end{table}

\textbf{Data and corruption.}
Both models are converted by supervised fine-tuning, without a separate
continued-pretraining stage, using the WeDLM block-diffusion
framework~\citep{liu2025wedlm} on the 5B-token corpus described in the main text.
Training examples are formatted as
conversations, packed into 4,096-token sequences, and supervised on the
completion tokens; prompt tokens provide context. Each 32-token block
receives an independently sampled masking probability in $[0,1]$, and the
denoiser losses use inverse-mask-probability weights normalized over
the local minibatch's supervised masked positions. The outcome actor
uses sequence-summed costs with one masked-token-count normalization,
as specified in Appendix~\ref{app:practical-supervision}.
The training recipe also includes a next-token prediction auxiliary loss
with the same relative weight as the diffusion loss.

\textbf{Optimization.}
Table~\ref{tab:training-configurations} lists the settings shared across the
two scales. Both use AdamW~\citep{loshchilov2017decoupled} and DeepSpeed
ZeRO-2. The 8B schedule runs for eight epochs (34{,}512 optimizer steps) and uses
activation checkpointing on every second recurrent-layer invocation;
the 1.7B configuration does not require it.
For latent-depth training, each rollout samples commitments from the
learned gate and feeds the committed target embeddings into later passes.
The outcome actor, first-pass control variate, and complete training
surrogate are given in Appendix~\ref{app:practical-supervision}.

\subsection{Generation quality and inference efficiency}
\label{app:evaluation-protocols}

\textbf{Generation quality.}
The benchmark comparison in Table~\ref{tab:main-results} follows the
OpenCompass protocol~\citep{2023opencompass} adopted by
SDAR~\citep{cheng2025sdar}. We apply each model's native chat template,
use greedy token selection, and cap generation at 4,096 tokens. ARC-Easy,
ARC-Challenge, MMLU, MMLU-Pro, and GPQA-Diamond measure general and
scientific reasoning; GSM8K and MATH-500 measure mathematical reasoning;
HumanEval, sanitized MBPP, and their EvalPlus variants measure code
generation. We report accuracy or exact match for non-code benchmarks
and execution-based pass@1 for code. The table's average is the unweighted
mean of its eleven benchmark scores.

For this quality comparison, each model pass commits one token using each method's default decoding rule. 
\modelname{} uses a window size of 16 and evaluates the mean cumulative halting criterion on the remaining positions, with $q=0.4$ at
1.7B and $q=0.5$ at 8B.
This preset disables entropy-based selection, so $\tau$ is inactive; a single
denoising step can still commit multiple tokens over successive recurrent passes.
Diffusion baselines use their single-token presets, including uncertainty-based
position selection where supported.

\textbf{Inference efficiency.}
Figures~\ref{fig:teaser-8b} and~\ref{fig:speed-pareto} use parallel commitment
and single-stream generation on one NVIDIA B200. \modelname{} and WeDLM use
their native CUDA-graph engines; Qwen3 uses vLLM~\citep{zheng2025vllm}.
SDAR uses JetEngine, LLaDA uses dInfer~\citep{ma2025dinfer},
Fast-dLLM-v2 uses its native engine, and Dream uses a compiled Fast-dLLM-v1
implementation with prefix caching~\citep{wu2025fastdllmtrainingfreeaccelerationdiffusion}.
Generation timing includes prefill and decoding, with device synchronization,
and excludes input tokenization and answer scoring.
Dataset throughput is the total generated-token count divided by the summed generation
time, rather than the mean of per-example rates. The $q,\tau$ sweeps use a
16-token window and position coefficient $\gamma=0.02$; the quality preset above
and this entropy-based parallel decoder are separate operating modes.

\textbf{Arithmetic cost.}
We estimate floating-point operations from execution records and model
dimensions. For each request, let $N_{\mathrm{layer}}$ count
token--Transformer-layer evaluations, $N_{\mathrm{head}}$ count
vocabulary readouts, and $N_{\mathrm{gen}}$ count generated tokens.
The estimated cost per generated token is
\begin{equation}
  \mathrm{FLOPs/token}
  =\frac{2P_{\mathrm{layer}}N_{\mathrm{layer}}
          +2P_{\mathrm{head}}N_{\mathrm{head}}
          +4H d_{\mathrm{head}}\bar L N_{\mathrm{layer}}}
         {N_{\mathrm{gen}}},
  \label{eq:experimental-flop-accounting}
\end{equation}
where $P_{\mathrm{layer}}$ counts dense projection parameters in one
Transformer layer, $P_{\mathrm{head}}$ counts output-head parameters,
$H$ is the number of query heads, $d_{\mathrm{head}}$ is the head dimension,
and $\bar L$ is approximated by the
prompt length plus half the generated length.
Dataset-level cost is total estimated FLOPs divided by total generated
tokens. Both \modelname{}-8B and WeDLM-8B have hidden dimension 4{,}096, feed-forward
dimension 12{,}288, 32 query heads, eight KV heads, head dimension 128,
and vocabulary size 151{,}936, giving approximately 192.9M dense
parameters per layer and 622.3M output-head parameters.
A multiply and an add count as two floating-point operations.

The estimate includes prefill, evaluated window positions, CUDA-graph
padding, and executed vocabulary readouts. For WeDLM, whole-forward
graph replay bypasses host-side layer-counter updates, so executed
work is reconstructed from recorded prefill lengths, decoding-step
counts, and graph shapes. The attention term uses the context-length
approximation above; minor elementwise operations are omitted.
For reference, a complete $K$-pass 8B denoising invocation evaluates
$10+K(16+10)$ Transformer layers per position: 62, 114, and 218 at
$K=2,4,8$, respectively.

\clearpage
\section{Case Studies}
\label{app:case-studies}

Aggregate depth statistics summarize computation across many responses.
Figure~\ref{fig:case-study-depth-examples} gives a complementary view of
three individual GSM8K generations from our decoder. Tokens committed at
different recurrent passes are interleaved within each response, illustrating
how commitment depth varies across positions.
The examples visualize these local depth assignments; the aggregate
allocation analysis is reported in Figure~\ref{fig:training-allocation}.

\begin{figure}[!ht]
  \centering
  \includegraphics[width=\textwidth]{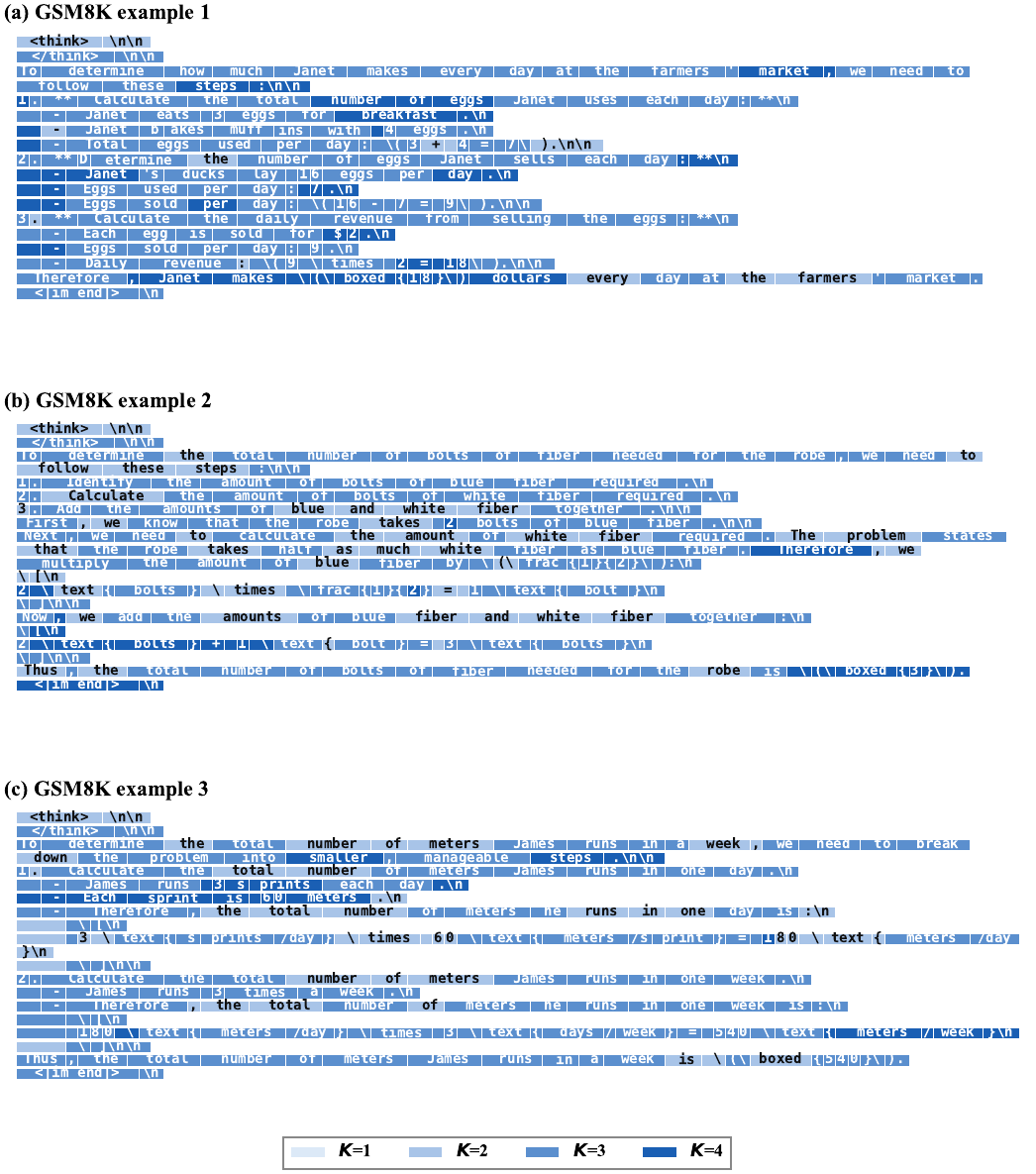}
  \caption{Token-wise recurrent depth in three GSM8K responses generated by
  \modelname{}. Each token is shaded by its recorded commitment pass;
  darker shades indicate later commitment within the corresponding denoising step.
  These examples use $K=4$,
  $q=0.5$, $\tau=0.4$, and a 16-token decoding window.}
  \label{fig:case-study-depth-examples}
\end{figure}
\FloatBarrier

\end{document}